\documentclass{article}
\PassOptionsToPackage{numbers, compress}{natbib}
\usepackage[main,final]{neurips_2026}
\usepackage[utf8]{inputenc}
\usepackage[T1]{fontenc}
\usepackage{hyperref}
\usepackage{url}
\usepackage{booktabs}
\usepackage{amsfonts}
\usepackage{nicefrac}
\usepackage{microtype}
\usepackage{xcolor}
\usepackage{subcaption}
\usepackage{titletoc}
\usepackage{tikz}
\usetikzlibrary{arrows.meta}
\usepackage{amsmath,mathtools,bm}
\usepackage{amsthm}
\usepackage[T1]{fontenc}
\usepackage{microtype}
\usepackage{enumitem}
\usepackage{booktabs,tabularx,array, makecell}
\hypersetup{colorlinks=true, linkcolor=blue!50!black, urlcolor=blue!50!black, citecolor=blue!50!black}
\newtheorem{theorem}{Theorem}
\newtheorem{proposition}[theorem]{Proposition}
\newtheorem{lemma}[theorem]{Lemma}
\newtheorem{corollary}[theorem]{Corollary}
\newtheorem{remark}[theorem]{Remark}
\newtheorem{definition}[theorem]{Definition}

\newcommand{\X}{\mathcal{S}}
\newcommand{\A}{\mathcal{A}}

\newcommand{\R}{\mathbb{R}}
\newcommand{\E}{\mathbb{E}}
\newcommand{\Prob}{\mathbb{P}}

\title{Quotient-Categorical Representations for Bellman-Compatible Average-Reward Distributional Reinforcement Learning}

\author{Ege C.~Kaya, Aliasghar~Pourghani, Vijay~Gupta, Abolfazl~Hashemi  \\
  Elmore Family School of Electrical and Computer Engineering\\
  Purdue University\\
  West Lafayette, IN 47906 \\
\texttt{\{kayae,apourgha,gupta869,abolfazl\}@purdue.edu}}

\begin{document}
\maketitle

\begin{abstract}
Average-reward reinforcement learning requires estimating the gain and a state-dependent bias, which solves the Poisson equation only up to an additive constant. Distributional analogues of the bias inherit this same ambiguity. We introduce a quotient-space formulation in which state-indexed laws are identified up to a common translation, together with a categorical parameterization that respects this symmetry. On this space, we define a projected average-reward distributional operator and show that it is well-defined, non-expansive in a coordinate Cram\'er metric, and admits fixed points. We interpret the projected fixed points through geometrically-stopped differential returns, whose categorical operators are contractive and whose fixed points converge in the Abel sense to a canonical fixed point of the projected undiscounted operator. We show that the means of the projected fixed points match the scalar bias up to an error caused by endpoint clipping of the categorical projections, and we bound this error. Finally, we establish almost-sure convergence and finite-iteration residual bounds for temporal-difference recursions with gain-centered rewards under i.i.d.\ and Markovian sampling. When the gain is unknown, we couple the recursion with an online gain estimator and prove the same guarantees under Markovian sampling, with a bound on the gain error.
\end{abstract}

\section{Introduction}

Distributional reinforcement learning (DRL) is well developed in the discounted setting, where the standard contraction arguments and Bellman-compatible return distributions support both analysis and algorithms \citep{bellemare2017distributional, bellemare2023distributional, rowland2018analysis, dabney2018distributional, dabney2018implicit}. The average-reward setting is structurally different. When the discounting is removed, the contraction geometry is lost, and fixed-policy evaluation is no longer described by a discounted return. Instead, average-reward evaluation revolves around two objects: the \emph{gain}, which is the long-run average reward, and the \emph{bias}, which solves the average-reward analogue of the Bellman equation, namely the Poisson equation. The bias of a state is the expected \emph{differential return} from that state, that is, the cumulative sum of gain-centered rewards. The bias distinguishes states and policies with the same gain, captures transient differential performance, and underlies average-reward evaluation and policy improvement \citep{puterman, sutton1998reinforcement, mahadevan1996average}. However, it is defined only up to an additive constant. This creates an obstruction for a distributional analogue, since a state-indexed family of laws on $\R$ has a fixed origin, yet translating these laws by the same constant should not change the object we learn. A second obstruction is that, unlike the discounted return, the differential return generally has no limiting law, since its fluctuations grow with the horizon due to the lack of discounting and its dampening effect. A distributional analogue of the bias must therefore rest on a different object, such as the geometrically stopped return.

We therefore model state-indexed law families modulo common translations, preserving the additive-constant symmetry of the Poisson equation in a finite-dimensional categorical representation. The learned object is a Bellman-consistent family of state-indexed laws whose means recover the scalar bias up to clipping. Each law arises as a limit of geometrically stopped differential-return laws and can therefore be read as describing centered rewards accumulated over a random horizon. This reading gives the laws meaning beyond their means, for instance through tail-risk measures such as CVaR, for which we provide approximation bounds. Our objective is complementary to the average-reward distributional criterion studied by \citet{rojas2026differential}, which refines the average-reward criterion itself through limiting per-step reward distributions. We instead refine the bias, the Bellman-compatible object that records state-dependent differential performance. Thus, our goal is not merely to describe long-run reward statistics, but to construct a distributional analogue of the scalar bias that can be used in average-reward dynamic-programming and temporal-difference (TD) learning arguments.

Part of the challenge is algorithmic. The projected maps we obtain are non-expansive rather than contractive, so discounted DRL contraction arguments do not apply. We show that the sampled categorical updates can be written as TD / stochastic approximation (SA) recursions in the classical sense \citep{robbins1951stochastic, sutton1988learning, tsitsiklis1994asynchronous, borkar1998asynchronous, borkar2008stochastic}, then prove that their mean-field maps are non-expansive and share fixed points with the projected differential-return operator. This connects our recursion to recent SA results for non-expansive maps \citep{bravo2024stochastic, blaser2026asymptotic}.

\noindent\textbf{Contributions.} Our main contributions are as follows.
\begin{enumerate}[leftmargin=*,labelsep=0.2em, align=left, itemsep=2pt]
\item We define exact and quotient-categorical projected operators on the quotient space of state-indexed law families modulo common translation, prove that both are well-defined and non-expansive, and show that the latter admits a fixed point.
\item We interpret projected fixed points through geometrically stopped returns, whose projected operators are contractive and whose fixed points converge as $\gamma\uparrow1$ to a distinguished fixed point of the projected undiscounted operator. We show that the means of the projected fixed points match the scalar bias up to an explicit clipping error due to categorical projections.
\item We prove almost-sure (a.s.) convergence and finite-iteration residual bounds for categorical TD recursions with gain-centered rewards under i.i.d.\ and Markovian sampling. When the gain is unknown, a coupled recursion with an online gain estimator attains the same Markovian guarantees, together with a bound on the gain error.
\end{enumerate}

\section{Related work}
DRL is built around Bellman-compatible return distributions in discounted MDPs, with categorical and quantile parameterizations providing the dominant practical representations \citep{bellemare2017distributional, rowland2018analysis, dabney2018distributional, dabney2018implicit}. Early theory studied the geometry, statistics, and relative advantages of categorical and quantile DRL \citep{rowland2018analysis, lyle2019comparative, rowland2019statistics}. Recent work has made many advancements by sharpening the statistical case for learning return distributions \citep{wang2023benefits, wang2024more}, providing non-asymptotic efficiency of distributional TD \citep{peng2024statistical, peng2025finite}, near-minimax guarantees in the generative-model regime \citep{rowland2024near}, establishing multivariate distributional dynamic programming \citep{wiltzer2024foundations}, offline distributional policy evaluation guarantees \citep{wu2023distributional}, and convergence analyses for KL-trained categorical updates \citep{kastner2025categorical}. In the average-reward setting, however, most work remains scalar, from early R-learning and average-reward foundations to modern off-policy, robust, and function-approximation results \citep{schwartz1993reinforcement,abounadi2001learning,mahadevan1996average, wan2021learning, wang2023model, vakili2024kernel, zurek2024span}. The closest prior average-reward DRL work is \citep{rojas2026differential}, where the authors study long-run per-step reward and differential return distributions via quantile representations. As opposed to these works, we seek a Bellman-compatible distributional analogue of the bias, which makes the quotient symmetry induced by additive-constant non-uniqueness central rather than incidental. On the algorithmic side, our finite-iteration analysis draws on fixed-point methods for non-expansive maps \citep{krasnosel1955two, mann1953mean, kim2007robustness,bravo2019rates,bravo2024stochastic,blaser2026asymptotic} and complements classical TD / SA analyses for contractive or linear settings \citep{dayan1994td,tsitsiklis1996analysis,bhandari2018finite,dalal2018finite,srikant2019finite,patil2023finite,qu2020finite,de2020fixed,chen2020finite,chen2022finite,chen2024lyapunov}.
\section{Preliminaries}\label{sec:prelim}

We work with a finite average-reward MDP \citep{puterman} $(\X, \A, R, P)$, where $\X = \{1,\dots,m\}$ is the finite set of states and $\A$ is the finite set of actions, $R: \X \times \A \to [0, 1]$ is the deterministic reward function and $P:\X \times \A \to \Delta(\X)$ is the transition kernel. A policy $\pi$ is fixed throughout, and we focus on its evaluation. With $\pi$ fixed, the induced process is a Markov reward process with transition matrix $P^\pi$, where $P^\pi_{ij} = \sum_{a\in\A}P(j \mid i, a) \,\pi(a \mid i)$ \citep{sutton1998reinforcement}. Throughout the paper, we assume that the induced Markov chain $\{S_t\}$ is irreducible and aperiodic, and write $\mu$ for its unique stationary distribution. We denote by $R_{ij}$ the finite-valued random one-step reward conditioned on $(S_t = i, S_{t+1} = j)$. We define the expected one-step reward vector by $r^\pi_i := \sum_{j\in \X} P^\pi_{ij} \E[R_{ij}]$ for all $i \in \X$. The (scalar) \emph{gain} or \emph{average reward} is then $\bar r^\pi := \sum_{i \in \X} \mu_i r^\pi_i$.

The \emph{bias} or \emph{relative value function} \citep{puterman} $v^\pi \in \R^m$ is any solution of the Poisson equation \citep{Meyn_Tweedie_Glynn_2009}
\begin{equation}\label{eq:poisson}
v^\pi = r^\pi - \bar r^\pi \mathbf 1 + P^\pi v^\pi.
\end{equation}
The Poisson equation is, in some sense, the analogue of the Bellman equation \citep{bellman1966dynamic} in the average-reward setting, where the discount factor has been set to $1$ and the rewards have been centered. It can equivalently be written as $(I- P^\pi) v^\pi = r^\pi - \bar r^\pi \mathbf 1$. Since the null space of $(I-P^\pi)$ contains any constant vector $c\mathbf 1 \in \R^m$, the solutions of \eqref{eq:poisson} are unique up to addition of a common constant, as the standing irreducibility assumption rules out any other null directions. This translation symmetry motivates the quotient-categorical formulation constructed in the rest of the work.

\section{Quotient-categorical bias estimation}\label{sec:quotient}

We begin our inquiry by seeking a distributional object which obeys two properties of the scalar bias: satisfying the gain-centered Poisson equation, and being defined only up to a common additive constant. Let $\mathcal F^\X := \{\eta = (\eta_i)_{i \in \X}: \eta_i \in \mathcal P_2(\R)\}$ denote the set of state-indexed collections of probability laws on $\R$ with finite second moment. We compare elements of $\mathcal F^\X$ with the Cramér metric. For $\xi,\zeta\in\mathcal P_2(\R)$ with CDFs $F_\xi,F_\zeta$, let $\ell_{\mathrm C}(\xi,\zeta):=(\int_\R(F_\xi(x)-F_\zeta(x))^2\,\mathrm dx)^{1/2}$, and for $\eta,\zeta\in\mathcal F^\X$, let $\ell_{\mathrm C,\infty}(\eta,\zeta):=\max_{i\in\X}\ell_{\mathrm C}(\eta_i,\zeta_i)$. For $c \in \R$, let $\tau_c(x) := x+c$, and define \emph{common-translation} on $\mathcal F^\X$ by $((\tau_c)_\# \eta)_i := (\tau_c)_\# \eta_i$, where $\#$ denotes pushforward. We also define the equivalence relation $\approx$ on any pair $\eta, \zeta \in \mathcal F^\X$ by $\eta \approx \zeta \iff \exists c\in\R \text{ such that } \zeta = (\tau_c)_\#\eta$,
and denote the quotient space by $\mathcal F^\X / \approx$. Let us write $B_{ij} := R_{ij} - \bar r^\pi$ for the centered one-step reward, and $\nu_{ij} := \mathrm{Law}(B_{ij})$. We now define the average-reward distributional operator.
\begin{definition}
The \emph{average-reward distributional operator} $\mathcal T: \mathcal F^\X \to \mathcal F^\X$ is defined componentwise:
\begin{equation}\label{eq:avg-operator}
(\mathcal T \eta)_i := \sum _{j\in \X} P^\pi_{ij}\, (\nu_{ij}\,\ast\,\eta_j), \quad \text{for all } i \in \X.
\end{equation}
\end{definition}
\begin{proposition}\label{prop:T}
For every $\eta \in \mathcal F^\X$ and $c \in \R$,
\begin{equation}\label{eq:commute}
\mathcal T\bigl((\tau_c)_\# \eta \bigr) = (\tau_c)_\# \mathcal T\eta.
\end{equation}
Consequently, $\mathcal T$ induces a well-defined operator $\overline{\mathcal T}([\eta]) :=[\mathcal T \eta]$ on the quotient space $\mathcal F^\X / \approx$.
\end{proposition}

The quotient formulation is instructive in seeing how the gain is not part of the quotient target itself. If we center rewards by an arbitrary constant $g \in \R$ instead of $\bar r^\pi$, the resulting state-indexed collection of laws differs only by a common translation. The following proposition formalizes this.
\begin{proposition}\label{prop:tg}
Let $g \in \R$. Define the \emph{$g$-centered average-reward distributional operator} $(\mathcal T_g \eta)_i := \sum_{j\in\X} P^\pi_{ij}(\nu_{ij}^{(g)}\ast \eta_j)$, where $\nu_{ij}^{(g)}:= \mathrm{Law}(R_{ij} -g)$, so that $\mathcal T = \mathcal T_{\bar r^\pi}$. Then, for any $g, g' \in \R$ and $\eta \in \mathcal F^\X$, $[\mathcal T_g \eta] = [\mathcal T_{g'} \eta]$ in $\mathcal F^\X / \approx$.
\end{proposition}

Guided by these observations, we introduce a categorical representation class which respects the same common-translation symmetry. Let $\Theta = \{ \theta_1 < \cdots < \theta_d\} \subset \R$ be an ordered support with constant stride $\varsigma:=\theta_{k+1} - \theta_k >0$ for $k=1,\dots, d-1$.\footnote{It is straightforward to use state-dependent supports and non-uniform stride, we forgo these for simplicity of exposition.} Let $\Delta_d$ denote the probability simplex in $\R^d$. For $p \in \Delta^\X_d$ and $c \in \R$, define the representative family of laws
\begin{equation}
\eta_i^{p, c} := \sum_{k=1}^d p_{i, k}\, \delta_{\theta_k +c},\quad \text{for all } i \in \X.
\end{equation}
For each $c \in \R$, define the $c$-shifted support $\Theta_c := \{\theta +c : \theta \in \Theta\}$. The state-indexed collection of coefficients $p$ therefore uniquely characterizes the following equivalence class of law--grid pairs:
\begin{equation}
\bigl[(\eta^{p, c}, \Theta_c)\bigr] := \bigl\{((\tau_a)_\# \eta^{p,c}, \Theta_{c+a}): a \in \R\bigr\}.
\end{equation}
In each pair within the equivalence class, the grid is translated together with the state-indexed laws, hence the coefficients $p$ encode the shape of each law but not its location on the real line. We abbreviate this class by $[\eta^p]$ when the accompanying grid is clear. We call this a \emph{quotient-categorical} representation: categorical because each state law is represented on the finite grid $\Theta$, and quotient because the common shift $c$ is not part of the state, but only a choice of representative.

The categorical projection of a point mass $x \in [\theta_k + c, \theta_{k+1} +c ]$ is 
\begin{equation}\label{eq:categorical-projection}
\Pi^{\Theta_c}_\mathrm C \delta_x = \frac{\theta_{k+1} + c - x}{\varsigma} \delta_{\theta_k + c} + \frac{x-(\theta_k + c)}{\varsigma} \delta_{\theta_{k+1} + c}.
\end{equation}
Mass below $\theta_1 + c$ or above $\theta_d + c$ is clipped at the nearest endpoint, and the projection extends linearly to probability laws. This is the standard categorical projection \citep{bellemare2017distributional} onto $\Theta_c$, applied over all states. We will show that this projection is shift-equivariant: first shifting a distribution by $c$ and then projecting onto the $c$-shifted categorical support is equivalent to projecting it onto the unshifted support and then shifting by $c$. Consequently, translating a law together with its grid leaves the coefficients of its projection unchanged.
\begin{lemma}[Projection shift-equivariance]\label{lem:proj-shift-equiv}
For every $\eta \in \mathcal P_2(\R)$ and $c \in \R$,
\begin{equation}
\Pi^{\Theta_c}_{\mathrm C} \bigl( (\tau_c)_\# \eta \bigr) = (\tau_c)_\# (\Pi^{\Theta}_{\mathrm C} \eta).
\end{equation}
\end{lemma}
This same fact is illustrated in Figure~\ref{fig:translation_fixed_support}. Every representative $\eta^{p,c}$ is supported on the $c$-shifted grid $\Theta_c$. For two such representatives with the same shift $c$, the CDF difference is constant between consecutive atoms, so the Cramér integral reduces to a finite sum of cumulative masses. For $p \in \Delta_d$, let $F_p(\theta_k) := \sum_{j=1}^k p_j$ for $ k=1,\dots,d-1$. We define, for $p, q \in \Delta_d$, the coordinate Cram\'er metric by
\begin{equation}
\ell^\Theta_{\mathrm C}(p, q)^2 := \varsigma \sum_{k=1}^{d-1}\bigl(F_p(\theta_k) - F_q(\theta_k)\bigr)^2.
\end{equation}
The statewise extension of this is the so-called supremum-Cram\'er metric \citep{bellemare2017distributional, bellemare2023distributional}, where for $p, q \in \Delta^\X_d$,
\begin{equation}
\ell^\Theta_{\mathrm C, \infty}(p, q) := \max_{i \in \X}\ell^\Theta_{\mathrm C}(p_i, q_i).
\end{equation}

The next proposition shows that these coordinate metrics coincide exactly with $\ell_{\mathrm C}$ and $\ell_{\mathrm C,\infty}$ on the corresponding representative families.
\begin{proposition}\label{prop:cramer}
For every $p, q \in \Delta^\X_d$ and $c \in \R$,
\begin{equation}
\ell_{\mathrm C}(\eta^{p, c}_i, \eta^{q, c}_i) = \ell^\Theta_{\mathrm C}(p_i, q_i),\quad\text{for all } i \in \X.
\end{equation}
Consequently,
\begin{equation}
\ell_{\mathrm C, \infty}(\eta^{p, c}, \eta^{q, c}) = \ell^\Theta_{\mathrm C, \infty}(p, q),\quad\text{for all } c \in \R.
\end{equation}
\end{proposition}

Proposition~\ref{prop:cramer} shows that the representative families $\eta^{p, c}$ faithfully encode the relevant Cram\'er geometry in the coefficient vectors $p \in \Delta^\X_d$, regardless of the shift $c$, as long as the same $c$ is chosen for both representatives. We now use this representation to define the projected analogue of the $g$-centered operator $\mathcal T_g$ on categorical quotient coordinates. Proposition~\ref{prop:tg} previously showed that, at the quotient level, $\mathcal T_g$ does not depend on the choice of $g$: changing $g$ only changes the representative of the output, but lands in the same equivalence class. After categorical projection, however, the representative matters, since the projection is performed relative to a fixed shifted support $\Theta_c$.

\begin{definition}
\label{defn:proj-operator}
Let $g \in \R$ and $p \in \Delta^\X_d$ and choose $c \in \R$ arbitrarily. Applying $\mathcal T_g$ to the representative family $\eta^{p, c}$ and projecting each component back onto $\Theta_c$ yields another categorical family on $\Theta_c$. We define $\mathcal G_g(p) = q \in \Delta^\X_d$ to be the unique element that satisfies
\begin{equation}
\Pi^{\Theta_c}_{\mathrm C}(\mathcal T_g \eta^{p, c})_i = \eta^{q, c}_i,\quad \text{for all } i \in \X.
\end{equation}
\end{definition}
We note that this definition is independent of the choice of the representative shift $c$. Indeed, since by definition $\eta^{p, c'} = (\tau_{c'-c})_\#\eta^{p, c}$, the same argument as in Proposition~\ref{prop:T} gives
\begin{equation}
\mathcal T_g \eta^{p, c'} = (\tau_{c' -c})_\# \mathcal T_g \eta^{p, c}.
\end{equation}

\begin{figure}[t]
    \centering
    % Figure 1 (redesign): registered translation vs. fixed-grid translation.
% Requires \usepackage{tikz} and \usetikzlibrary{arrows.meta}.
% Drop-in: \input{fig1_tikz} inside a figure environment (width ~ 5.4in = \textwidth of NeurIPS).
% Stems are drawn to scale: height = \H * probability mass.
\begingroup
\definecolor{pcol}{HTML}{6A4FA0}% coefficient vector p
\definecolor{qcol}{HTML}{D0479C}% coefficient vector q
\definecolor{gridcol}{HTML}{3F345B}
\definecolor{faint}{HTML}{B59BC8}
\begin{tikzpicture}[x=0.80cm,y=1cm,
    font=\footnotesize,
    axis/.style={gridcol, line width=0.5pt},
    tick/.style={gridcol, line width=0.5pt},
    stemp/.style={pcol, line width=1.6pt},
    stemq/.style={qcol, line width=1.6pt},
    ghost/.style={pcol!45, line width=0.9pt, dash pattern=on 1.6pt off 1.1pt},
    guide/.style={faint, line width=0.35pt, dash pattern=on 0.8pt off 1.2pt},
    split/.style={pcol!70!black, line width=0.4pt, -{Stealth[length=2.3pt]}},
    lab/.style={gridcol, font=\scriptsize},
  ]
  \def\H{2.3}   % cm per unit mass
  \def\c{0.62}  % common shift (in units of the stride)
  \def\ytop{1.8}
  \def\plist{0/0.08,1/0.22,2/0.38,3/0.23,4/0.09}
  % q = projection of the shifted atoms onto the fixed grid (computed for c = 0.62 strides)
  \def\qlist{0/0.0304,1/0.1332,2/0.2808,3/0.3230,4/0.2326}

  % ---------------- panel (a) ----------------
  \begin{scope}
    \node[font=\small\bfseries] at (2.0,3.0) {(a) Law and grid shifted together};
    % top row: eta^{p,0} on Theta
    \draw[axis] (-0.35,\ytop) -- (4.35,\ytop);
    \foreach \k/\pk [count=\n] in \plist {
      \draw[tick] (\k,\ytop-0.05) -- (\k,\ytop+0.05);
      \node[lab, below] at (\k,\ytop-0.04) {$\theta_{\n}$};
      \draw[stemp] (\k,\ytop) -- (\k,\ytop+\H*\pk);
      \fill[pcol] (\k,\ytop+\H*\pk) circle (1.5pt);
    }
    \node[anchor=east, pcol] at (-0.45,\ytop+0.22) {$\eta^{p,0}$};
    % guides from the original grid positions
    \foreach \k in {0,...,4} \draw[guide] (\k,\ytop-0.42) -- (\k,0);
    % bottom row: translated law on the translated grid
    \draw[axis] (-0.35+\c,0) -- (4.35+\c,0);
    \foreach \k/\pk [count=\n] in \plist {
      \draw[tick] (\k+\c,-0.05) -- (\k+\c,0.05);
      \node[lab, below] at (\k+\c,-0.04) {$\theta_{\n}{+}c$};
      \draw[stemp] (\k+\c,0) -- (\k+\c,\H*\pk);
      \fill[pcol] (\k+\c,\H*\pk) circle (1.5pt);
    }
    \node[anchor=east, pcol, align=right, inner sep=1pt] at (-0.45,0.25) {$(\tau_c)_{\#}\eta^{p,0}$\\[-1pt]$=\eta^{p,c}$};
    % offset marker
    \draw[gridcol, line width=0.4pt, {Stealth[length=2.4pt]}-{Stealth[length=2.4pt]}] (0,0.42) -- (\c,0.42);
    \node[lab, above] at (0.5*\c,0.40) {$c$};
    % readout
    \node[pcol] at (2.0,-0.64) {coefficients unchanged: $p=(.08,.22,.38,.23,.09)$};
  \end{scope}

  % ---------------- panel (b) ----------------
  \begin{scope}[xshift=6.75cm]
    \node[font=\small\bfseries] at (2.0,3.0) {(b) Law shifted, grid fixed};
    \draw[axis] (-0.35,\ytop) -- (4.35,\ytop);
    \foreach \k/\pk [count=\n] in \plist {
      \draw[tick] (\k,\ytop-0.05) -- (\k,\ytop+0.05);
      \node[lab, below] at (\k,\ytop-0.04) {$\theta_{\n}$};
      \draw[stemp] (\k,\ytop) -- (\k,\ytop+\H*\pk);
      \fill[pcol] (\k,\ytop+\H*\pk) circle (1.5pt);
    }
    \node[anchor=east, pcol] at (-0.45,\ytop+0.22) {$\eta^{p,0}$};
    % bottom row: same (fixed) grid; translated atoms (ghosts) and their projection
    \draw[axis] (-0.35,0) -- (4.35,0);
    \foreach \k/\pk in \plist {
      \draw[ghost] (\k+\c,0) -- (\k+\c,\H*\pk);
      \draw[pcol!55, line width=0.6pt, fill=white] (\k+\c,\H*\pk) circle (1.4pt);
    }
    \foreach \k/\qk [count=\n] in \qlist {
      \draw[tick] (\k,-0.05) -- (\k,0.05);
      \node[lab, below] at (\k,-0.04) {$\theta_{\n}$};
      \draw[stemq] (\k,0) -- (\k,\H*\qk);
      \fill[qcol] (\k,\H*\qk) circle (1.5pt);
    }
    \node[anchor=east, qcol] at (-0.45,0.22) {$\Pi^{\Theta}_{\mathrm C}(\tau_c)_{\#}\eta^{p,0}$};
    % callout for the off-grid atoms
    \draw[faint, line width=0.4pt] (\c+0.05,\H*0.08+0.07) to[out=95,in=-90] (0.45,1.06);
    \node[lab, text=pcol!80!black, anchor=south west, inner sep=1pt] at (0.05,1.06) {off-grid atoms};
    % one worked split: the atom at theta_3 + c (mass .38) is split .38 / .62 between theta_3 and theta_4
    \draw[split] (2+\c-0.07,\H*0.38+0.02) to[out=160,in=60] (2.04,\H*0.2808+0.07);
    \draw[split] (2+\c+0.07,\H*0.38+0.02) to[out=20,in=120] (2.96,\H*0.3230+0.07);
    \node[lab, text=pcol!80!black, anchor=east, inner sep=1pt] at (1.94,0.80) {$\times.38$};
    \node[lab, text=pcol!80!black, anchor=west, inner sep=1pt] at (3.02,0.86) {$\times.62$};
    % clipping beyond the last atom
    \draw[qcol, line width=0.45pt, -{Stealth[length=2.6pt]}] (4+\c,\H*0.09+0.07) to[out=90,in=20] (4.08,\H*0.2326+0.05);
    \node[lab, text=qcol, anchor=west, inner sep=1pt] at (4.74,0.36) {clipped to $\theta_5$};
    % readout
    \node[qcol] at (2.0,-0.64) {coefficients changed: $q=(.03,.13,.28,.32,.23)$};
  \end{scope}
\end{tikzpicture}
\endgroup
    \caption{
    \textbf{(a)} Shifting a law together with its grid leaves its coefficients $p$ unchanged (Lemma~\ref{lem:proj-shift-equiv}). \textbf{(b)} Shifting the law against a fixed grid moves its atoms off the grid. Projection then splits each atom between its two neighbors and clips mass beyond the last atom, giving $q\neq p$.
    }
    \label{fig:translation_fixed_support}
\end{figure}

The application of Lemma~\ref{lem:proj-shift-equiv} shows that projection onto $\Theta_{c'}$ or onto $\Theta_c$ produces categorical families with the same coefficient vector $q$. Therefore, $\mathcal G_g$ is a well-defined operator on $\Delta^\X_d$ for any $g \in \R$. Unlike $\mathcal T$, which acts on $\mathcal F^\X/\approx$, the projected operator acts on coefficients, because translates of the same law projected onto a fixed grid generally have different coefficients (Figure~\ref{fig:translation_fixed_support}(b)). In particular, changing $g$ translates the backup $\mathcal T_g\eta^{p,c}$ against the fixed grid $\Theta_c$, so $\mathcal G_g$ depends on $g$ even though $[\mathcal T_g\eta]$ does not (Proposition~\ref{prop:tg}). We therefore single out $\mathcal G := \mathcal G_{\bar r^\pi}$ as the ``correctly-centered’’ projected operator. As in discounted categorical TD learning \citep{bellemare2017distributional}, $\mathcal G$ is a projected Bellman backup, here non-expansive rather than contractive.

We now state the main theorem of this section, establishing the non-expansiveness of $\mathcal G_g$.

\begin{theorem}[Non-expansiveness of $\mathcal G_g$]\label{thm:gg-nonexp}
For every $g \in \R$, and every $p, q \in \Delta^\X_d$,
\begin{equation}
\ell_{\mathrm C, \infty}^\Theta \bigl(\mathcal G_g(p), \mathcal G_g(q)\bigr) \le \ell_{\mathrm C, \infty}^\Theta(p, q).
\end{equation}
In particular, the correctly-centered operator $\mathcal G$ is non-expansive.
\end{theorem}
\begin{proof}[Proof sketch]
Fix a common representative shift $c$ and compare the representative families $\eta^{p, c}$ and $\eta^{q, c}$. For each state $i$, the $i$th component of $\mathcal T_g$ is a $P^\pi_{ij}$-weighted average of convolved laws $\nu_{ij}^{(g)} \, \ast\, \eta_j$. Convolution with a fixed law is non-expansive in the Cram\'er metric, and the subsequent $P^\pi_{ij}$-averaging is also non-expansive. Thus $\mathcal T_g$ is non-expansive in $\ell_{\mathrm C, \infty}$. Composing with the standard non-expansive categorical projection and then using Proposition~\ref{prop:cramer} yields the claim.
\end{proof}
Although the exact operator $\mathcal T_g$ need not have a fixed point in $\mathcal F^\X$, the projected operator $\mathcal G_g$ is a continuous map on the compact set $\Delta^\X_d$ and therefore always has one.
\begin{corollary}\label{corol:fixed-points}
For every $g\in\R$, $\mathrm{fix}(\mathcal G_g)$ is nonempty. In particular, $\mathcal G$ has a fixed point.
\end{corollary}
The full proofs of Theorem~\ref{thm:gg-nonexp} and Corollary~\ref{corol:fixed-points} are given in Appendix~\ref{app:A}.

\subsection{Interpreting the projected fixed point}\label{sec:meaning}
$\mathcal T_g$ need not have a fixed point, but we can still interpret the fixed points of $\mathcal G$ through differential returns stopped at a random time. Fix $\gamma\in(0,1)$ and let $N_\gamma$ be independent of the chain with $\Prob(N_\gamma=n)=(1-\gamma)\gamma^{n-1}$ for $n\ge1$. For $S_0=i$, let $R_t$ denote the reward on the transition from $S_t$ to $S_{t+1}$, let the geometrically stopped differential return be
\begin{equation}\label{eq:geometric-return}
Y_i^\gamma:=\sum_{t=0}^{N_\gamma-1}(R_t-\bar r^\pi),
\end{equation}
and let $\eta_i^\gamma := \mathrm{Law}(Y_i^\gamma)$. Without loss of generality, we assume $0 \in \Theta$.\footnote{Since $\mathcal G$ does not depend on the representative shift $c$, any atom $\theta_\circ\in\Theta$ can serve as the origin by taking $c=-\theta_\circ$.} The corresponding exact and projected operators are
\begin{equation}\label{eq:stopped-operators-main}
(\mathcal T_\gamma\eta)_i:=\sum_{j\in\X}P^\pi_{ij}\,\nu_{ij}\ast\bigl((1-\gamma)\delta_0+\gamma\eta_j\bigr),
\qquad
\mathcal G_\gamma(p):=(1-\gamma)\mathcal G(p^0)+\gamma\mathcal G(p),
\end{equation}
where $p^0\in\Delta^\X_d$ puts unit mass on the atom $0\in\Theta$ in every state. After each centered reward, both operators stop with probability $1-\gamma$ and continue with probability $\gamma$. Finally, let $v^\circ:=\sum_{t=0}^\infty (P^\pi)^t\bigl(r^\pi-\bar r^\pi\mathbf 1\bigr)$
denote the solution of the Poisson equation with $\mu^\top v^\circ=0$.

\begin{theorem}[Stopped returns and the Abel limit]\label{thm:stopped-main}
Let $\gamma\in(0,1)$. The operators $\mathcal T_\gamma$ and $\mathcal G_\gamma$ are $\gamma$-contractive in $\ell_{\mathrm C,\infty}$ and $\ell^\Theta_{\mathrm C,\infty}$, respectively. The fixed point of $\mathcal T_\gamma$ is $\eta^\gamma$, and $\E[Y^\gamma]\to v^\circ$ as $\gamma\uparrow1$. The fixed point $p_\gamma$ of $\mathcal G_\gamma$ satisfies
\begin{equation}\label{eq:abel-main}
p_\gamma\to p^\dagger:=\lim_{N\to\infty}\frac1N\sum_{n=1}^N\mathcal G^n(p^0)\in\mathrm{fix}(\mathcal G)\quad\text{as }\gamma\uparrow1.
\end{equation}
\end{theorem}
For $\gamma<1$, the projected fixed point $p_\gamma$ therefore approximates the law of a randomly stopped differential return, and Proposition~\ref{prop:stopped-approximation} bounds this approximation error in the Cram\'er metric, $W_1$, and CVaR. As $\gamma\uparrow1$, these approximations select the fixed point $p^\dagger$ of $\mathcal G$. Theorem~\ref{thm:geometric-target} also gives a closed form for $p_\gamma$ and a criterion for the uniqueness of the fixed point of $\mathcal G$.

We next relate the means of the projected fixed points to the scalar bias. For the rest of this section, let $\Theta=\{-M,-M+\varsigma,\ldots,M\}$ with $M\ge1$. For $p\in\Delta^\X_d$, let $m(p)$ have entries $m_i(p):=\sum_{k=1}^dp_{i,k}\theta_k$, the mean of $\eta^{p,0}_i$. Categorical projection preserves the mean of every law supported in $[-M,M]$, so it can change means only through clipping at the endpoints. With $[x]_{-M}^M:=\max\{-M,\min\{x,M\}\}$, this change is measured by the clipping error
\begin{equation}\label{eq:clipping-defect}
\beta_i(p):=\sum_{j\in\X}P^\pi_{ij}\E\bigl[[B_{ij}+X_j]_{-M}^M-(B_{ij}+X_j)\bigr],\qquad X_j\sim\eta^{p_j,0},
\end{equation}
which vanishes when no mass is clipped.
\begin{proposition}[Mean recovery under clipping]\label{prop:mean-main}
Every $p^\star\in\mathrm{fix}(\mathcal G)$ satisfies
\begin{equation}\label{eq:perturbed-poisson-main}
(I-P^\pi)m(p^\star)=r^\pi-\bar r^\pi\mathbf1+\beta(p^\star).
\end{equation}
Moreover, there are constants $C_1,C_2$ that depend on the Markov reward process, but not on the grid, such that every $p\in\Delta^\X_d$ satisfies
\begin{equation}\label{eq:headline-grid-error}
\inf_{c\in\R}\|m(p)-v^\pi-c\mathbf1\|_\infty\le C_1\bigl(M^{-1}+d^{-1}\bigr)+C_2\sqrt M\,\ell^\Theta_{\mathrm C,\infty}\bigl(p,\mathcal G(p)\bigr).
\end{equation}
\end{proposition}
For a fixed point, the residual term vanishes and the mean error is $O(M^{-1}+d^{-1})$. The bound also turns the residual guarantees of Sections~\ref{sec:centered} and~\ref{sec:uncentered} into guarantees for the scalar bias. Appendix~\ref{app:meaning} provides the full statements and proofs, including the explicit constants for $C_1$ and $C_2$ in Proposition~\ref{prop:poisson-consistency}.

\section{A centered stochastic approximation recursion}\label{sec:centered}
We now turn from the exact projected backup $\mathcal G$, a non-expansive map with at least one fixed point, to an SA thereof, based on sampled transitions. Customarily, SA in RL comes as a TD update, where we observe only one transition at a time and asynchronously modify only one state block.

A one-state sample is defined to be a tuple $y = (s, b, s') \in \mathcal Y := \X \times [-1, 1] \times \X$.\footnote{Note that rewards lying in $[0, 1]$ implies that centered rewards lie in $[-1, 1]$.} For $p \in \Delta^\X_d$, we define $H(p ,y) = q \in \Delta^\X_d$, the one-sample backup map, by
\begin{equation}\label{eq:one-sample-backup}
q_u = \begin{cases}
p_u, &u\ne s,\\
\tilde q, &u=s,
\end{cases}
\end{equation}
where $\tilde q \in \Delta_d$ is the unique coefficient vector satisfying $\Pi^\Theta_{\mathrm C}\bigl( (\tau_b)_\# \eta ^{p_{s'}, 0}\bigr) = \eta^{\tilde q, 0}$. The asynchronous recursion is thus the well-known stochastic Krasnosel'ski\u{\i}--Mann (SKM) iteration \citep{krasnosel1955two, mann1953mean} 
\begin{equation}\label{eq:KM}
p_{k+1} = p_k + \alpha_k \bigl(H(p_k, Y_k) - p_k\bigr),\quad k\ge0.
\end{equation}

The next proposition rewrites the $g$-centered projected operator $\mathcal G_g$ blockwise in categorical coefficients, which is the form needed to derive asynchronous SA updates. Before the statement, we introduce a slight abuse of notation. For $u \in \Delta_d$ and $c \in \R$, we write $
\eta^{u, c} := \sum_{k=1}^d u_k \delta_{\theta_k +c}$.
In particular, $(\eta^{p, c})_i = \eta^{p_i, c}$ for $p \in \Delta^\X_d$.
\begin{proposition}\label{prop:g-block-form}
Let $g \in \R$, $p \in \Delta^\X_d$, and $i \in \X$. The $i$th block of $\mathcal G_g(p)$ is the unique vector $q_i \in \Delta_d$ such that
\begin{equation}
\eta^{q_i, 0} = \Pi^\Theta_{\mathrm C}\Bigl(\sum_{j \in \X} P^\pi_{ij}(\nu_{ij}^{(g)} \,\ast\, \eta^{p_j, 0})\Bigr).
\end{equation}
\end{proposition}
The next lemma simply states that the one-sample backup map is also a non-expansion. 
\begin{lemma}\label{lem:one-sample-nonexp}
For every $y \in \mathcal Y$ and $p, q \in \Delta^\X_d$,
\begin{equation}
\ell^\Theta_{\mathrm C, \infty}\bigl( H(p, y), H(q, y)\bigr) \le \ell^\Theta_{\mathrm C, \infty}(p, q).
\end{equation}
\end{lemma}

We will now work on creating an asynchronous relaxation of $\mathcal G$ through the use of the one-sample backup map $H$. Let $\rho \in \Delta(\X)$ such that $\rho_{\min}:=\min_{i \in \X} \rho_i >0$. We define
\begin{equation}
D_\rho := \mathrm{diag}(\rho_1, \dots, \rho_m),\qquad h_\rho(p):= p + D_\rho\bigl(\mathcal G(p) - p\bigr).
\end{equation}
$h_\rho$ is the \emph{mean-field map} of the one-sample stochastic update under state-sampling law $\rho$, in the standard SA sense. Of particular interest will be when $\rho=\mu$, the stationary distribution of the Markov chain under $\pi$. We have the following favorable results on this averaged map.

\begin{proposition}\label{prop:mean-field-fix}
Suppose $Y = (S, B, S')$ is sampled such that $S \sim \rho$ and
\begin{equation}
\mathbb P(S'= j \mid S = i) = P^\pi_{ij}, \qquad B \mid (S=i, S'=j) \sim \nu_{ij}\quad \text{for all } i, j \in \X.
\end{equation}
Then, $\E\bigl[H(p, Y)\bigr] = h_\rho(p)$ for all $p \in \Delta^\X_d$, $h_\rho$ is non-expansive in $\ell^\Theta_{\mathrm C, \infty}$, and $\mathrm{fix}(h_\rho) = \mathrm{fix}(\mathcal G)$. In particular, the same conclusions hold with $\rho = \mu$ under the standing irreducibility assumption.
\end{proposition}

The next corollary will be useful in translating residual bounds for $h_\rho$ into residual bounds for $\mathcal G$.
\begin{corollary}\label{corol:residual-conversion}
For every $p \in \Delta^\X_d$,
\begin{equation}
\rho_{\min} \ell^\Theta_{\mathrm C, \infty}\bigl(p, \mathcal G(p)\bigr) \le \ell^\Theta_{\mathrm C, \infty}\bigl(p, h_\rho(p)\bigr) \le \ell^\Theta_{\mathrm C, \infty}\bigl(p, \mathcal G(p)\bigr).
\end{equation}
\end{corollary}
Proofs of Proposition~\ref{prop:g-block-form}, Lemma~\ref{lem:one-sample-nonexp}, Proposition~\ref{prop:mean-field-fix}, and Corollary~\ref{corol:residual-conversion} are given in Appendix~\ref{app:B}.

We state the two main SA results for the centered-reward recursion approximating $\mathcal G$. The i.i.d.\ regime is the relevant model for settings such as replay-buffer sampling or a generative simulator, where we can approximately draw independent one-step samples from a fixed sampling law. The Markovian regime captures online asynchronous TD-learning along a single trajectory of the policy-induced chain. In both settings, the quantity we control is the mean-field residual $\ell^\Theta_{\mathrm C, \infty}\bigl(p_k, h_{\rho}(p_k)\bigr)$. This is the standard quantity in non-expansive SA \citep{mann1953mean, bravo2019rates, bravo2024stochastic, blaser2026asymptotic, kim2007robustness}, where fixed points need not be unique.

\begin{theorem}[Centered recursion i.i.d.\ convergence]\label{thm:centered-iid}
Suppose the samples $Y_k = (S_k, B_k, S'_k)$ are i.i.d., that $S_k \sim \rho$ with $\rho_{\min} >0$ and that 
\begin{equation}
\mathbb P (S'_k = j \mid S_k = i) = P^\pi_{ij}, \qquad B_k \mid (S_k =i, S'_k=j) \sim \nu_{ij}\quad \text{for all } i, j \in \X.
\end{equation}
The following two conclusions hold for the centered SKM recursion \eqref{eq:KM}. First, if $\alpha_k = (k+1)^{-a}$ with $\frac{2}{3} < a \le 1$, then $p_k$ converges a.s.\ to a random fixed point $p^\star \in \mathrm{fix}(h_\rho) = \mathrm{fix}(\mathcal G)$. Second, fix any $a_1 \in \bigl(\frac{2}{3}, 1\bigr)$ and set $\varepsilon_{a_1} := (3a_1 -2)/6$. There is an explicit two-phase step size schedule and constants $C_\mathrm{iid}, \kappa_{a_1}>0$ such that, for all $k \ge 1$,
\begin{equation}
\E\Bigl[\ell^\Theta_{\mathrm C, \infty}\bigl(p_k, \mathcal G (p_k)\bigr) \Bigr] \le \frac{C_{\mathrm{iid}}}{\rho_{\min}} (k+1)^{-1/6} \min\left\{ \kappa_{a_1}\log(k+1), (k+1)^{\varepsilon_{a_1}}\right\},
\end{equation}
where $C_{\mathrm{iid}}$ depends only on $a_1$, $\lvert \X \rvert$, and $\mathrm{diam}_{\ell^\Theta_{\mathrm C,\infty}}(\Delta^\X_d) = \sqrt{\theta_d - \theta_1}$.
The full fixed-exponent rate profile is presented in Appendix~\ref{app:B}. In particular, choosing $a = \frac{2}{3} + \epsilon$ for $\epsilon >0$ gives a.s.\ convergence with expected residual decaying at rate $O(k^{-1/6+\epsilon/2})$.
\end{theorem}
\begin{theorem}[Centered recursion Markovian convergence]\label{thm:centered-markov}
Let $(S_k)_{k\ge 0}$ be the policy-induced Markov chain, which is irreducible and aperiodic with stationary distribution $\mu$, and let $Y_k = (S_k, B_k, S_{k+1})$ be the associated centered one-sample process. The following two conclusions hold for the centered SKM recursion \eqref{eq:KM}. First, if $\alpha_k = (k+1)^{-a}$ with $\frac{4}{5} < a \le 1$, then $p_k$ converges a.s.\ to a random fixed point $p^\star \in \mathrm{fix}(h_\mu)=\mathrm{fix}(\mathcal G)$. Second, fix any $a_1 \in \bigl(\frac{4}{5}, 1\bigr)$ and set $\varepsilon_{a_1}:=(5a_1-4)/10$. There is an explicit two-phase step size schedule and constants $C_{\mathrm{mk}}, \kappa_{a_1} >0$ such that, for all $k \ge 1$, 
\begin{equation}
\E\Bigl[\ell^\Theta_{\mathrm C, \infty}\bigl(p_k, \mathcal G (p_k)\bigr) \Bigr] \le \frac{C_{\mathrm{mk}}}{\mu_{\min}}(k+1)^{-1/10} \min \left\{\kappa_{a_1}\log(k+1), (k+1)^{\varepsilon_{a_1}} \right\},
\end{equation}
where $C_{\mathrm{mk}}$ depends on $a_1$, $\lvert \X\rvert$, $\mathrm{diam}_{\ell^\Theta_{\mathrm C,\infty}}(\Delta^\X_d)=\sqrt{\theta_d - \theta_1}$, and a finite-state Poisson constant for $(Y_k)$, equivalently expressable through the norm of its fundamental matrix. The full fixed-exponent rate profile is presented in Appendix~\ref{app:B}. In particular, choosing $a = \frac{4}{5}+\epsilon$ for  $\epsilon >0$ gives a.s.\ convergence with expected residual decaying at rate $O(k^{-1/10+\epsilon/2})$.
\end{theorem}
\begin{proof}[Proof sketch for Theorems~\ref{thm:centered-iid} and~\ref{thm:centered-markov}]
By Lemma~\ref{lem:one-sample-nonexp}, the one-sample backup map is non-expansive. By Proposition~\ref{prop:mean-field-fix} and Corollary~\ref{corol:fixed-points}, the mean-field map is non-expansive with a non-empty set of fixed points. Compactness of $\Delta^\X_d$ yields bounded iterates and bounded martingale noise, so the i.i.d.\ recursion falls under the bounded-noise non-expansive SA results of \citet{bravo2024stochastic}. In the Markovian regime, finite $\X$ and $\A$ and the deterministic reward $R(s, a)$ make $(Y_k)_{k\ge0}$ a finite-state irreducible and aperiodic chain, so the fixed-exponent results of \citet{blaser2026asymptotic} apply. The two-phase residual bounds are proved separately in Appendix~\ref{app:B}.
\end{proof}

These bounds use a two-phase step size. A fixed supercritical exponent gives a.s.\ convergence and a logarithm-free residual bound but loses a polynomial factor, whereas the critical exponent gives the sharper power of $k$ up to a logarithm. The two-phase schedule tracks the smaller of the two. Its switch point depends only on $a_1$ but can be very late, so over practical horizons the fixed-exponent bound applies. The powers $k^{-1/6}$ and $k^{-1/10}$ come from general non-expansive SA and are neither lower bounds nor intrinsic categorical rates. Faster instance-dependent rates may follow from anchored variance-reduced methods \citep{kaya2026banachspacetheorymarkovianhalpern}.

\section{A coupled recursion for when the gain is unknown}\label{sec:uncentered}

The centered recursion isolates the non-expansive structure, but it is not directly implementable unless the gain $\bar r^\pi$ is known. In practice, we observe uncentered rewards $R$, not $B:= R - \bar r^\pi$. 

\begin{proposition}\label{prop:counterexample}
In general, replacing the centered reward $B = R - \bar r^\pi$ by the raw reward $R$ in the one-sample backup \eqref{eq:one-sample-backup} does not yield an SA of $\mathcal G$.
\end{proposition}

The example is proved in Appendix~\ref{app:C}. We now replace centered rewards by uncentered ones, and augment the quotient-categorical update with an online gain estimate $g$ on $\mathcal Z := \Delta^\X_d \times [0, 1]$. We show that the augmented update is non-expansive, identify its mean-field map and its fixed points, and relate its residual to $\ell^\Theta_{\mathrm C, \infty}\bigl(p, \mathcal G(p)\bigr)$ and the gain error $\lvert g - \bar r^\pi\rvert$. For $g \in \R$ and an uncentered sample $y = (s, r, s') \in \X \times [0, 1] \times \X$, define the gain-parameterized one-sample backup map $H_g(p, y) = q$, where
\begin{equation}\label{eq:one-sample-backup-gain}
q_u = \begin{cases}
p_u, &u\ne s,\\
\tilde q, &u=s,
\end{cases}
\end{equation}
where $\tilde q \in \Delta_d$ is the unique coefficient vector satisfying $\Pi^\Theta_{\mathrm C}\bigl( (\tau_{r-g})_\# \eta ^{p_{s'}, 0}\bigr) = \eta^{\tilde q, 0}$. The corresponding mean-field map is defined in terms of the stationary state distribution $\mu$ as $h^{(g)}_\mu(p) := p + D_\mu(\mathcal G_g\bigl(p)-p\bigr)$. Throughout the section, fix $\lambda \ge \varsigma^{-1/2}$, and define the product metric
\begin{equation}\label{eq:product-metric}
d_\lambda\bigl((p, g), (q, g')\bigr) := \ell^\Theta_{\mathrm C, \infty}(p, q) + \lambda \lvert g- g'\rvert.
\end{equation}
We will now use a coupled SKM recursion with the same step size, one for the categorical coefficients, and the other for the gain:
\begin{equation}\label{eq:coupled-skm}
\begin{gathered}
p_{k+1} = p_k + \alpha_k\bigl(H_{g_k}(p_k, Y_k) - p_k\bigr),\\
g_{k+1} = g_k + \alpha_k(R_k - g_k),
\end{gathered}
\end{equation}
where $Y_k = (S_k, R_k, S_{k+1})$. The gain recursion is a running average of the observed rewards, used to center each sampled reward. We next show that the augmented update is non-expansive in $d_\lambda$.
\begin{proposition}\label{prop:uncentered-nonexpansion}
For every uncentered sample $y = (s, r, s') \in \X \times [0, 1] \times \X$, and $(p, g), (q, g') \in \mathcal Z$,
\begin{equation}
d_\lambda\bigl( (H_g(p, y), r), (H_{g'}(q, y), r)\bigr) \le d_\lambda\bigl((p, g), (q, g')\bigr).
\end{equation}
\end{proposition}

The next proposition identifies the mean-field map associated with the stationary distribution $\mu$. 

\begin{proposition}\label{prop:mean-field-online}
Let $Y= (S, R, S')$ be a one-step sample from the stationary distribution of the policy-induced chain, i.e.,
\begin{equation}
S \sim \mu,\qquad \mathbb P(S'=j \mid S=i) = P^\pi_{ij},\qquad R \mid (S=i, S'=j) \sim \mathrm{Law}(R_{ij})
\end{equation}
for all $i, j \in \X$. Then
\begin{equation}
\E\bigl[\bigl(H_g(p, Y), R\bigr)\bigr] = \bigl(h^{(g)}_\mu(p), \bar r^\pi\bigr)\quad \text{for all } (p,g) \in \mathcal Z.
\end{equation}
Moreover, the map $(p, g) \mapsto\bigl(h^{(g)}_\mu(p), \bar r^\pi \bigr)$ is non-expansive in $d_\lambda$, and $\mathrm{fix}\bigl((p, g) \mapsto\bigl(h^{(g)}_\mu(p), \bar r^\pi \bigr)\bigr) = \{(p, \bar r^\pi): p \in \mathrm{fix}(\mathcal G)\}$.
\end{proposition}
\begin{corollary}\label{corol:residual-conv-online}
For every $(p, g) \in \mathcal Z$,
\begin{equation}
\mu_{\min}\ell^\Theta_{\mathrm C, \infty}\bigl(p, \mathcal G(p)\bigr) \le d_\lambda \bigl(\bigl(p, g\bigr), \bigl(h^{(g)}_\mu(p), \bar r^\pi\bigr) \bigr)\quad\text{and}\quad \lvert g- \bar r^\pi \rvert\le \lambda^{-1} d_\lambda \bigl(\bigl(p, g\bigr), \bigl(h^{(g)}_\mu(p), \bar r^\pi\bigr) \bigr).
\end{equation}
\end{corollary}

Appendix~\ref{app:C} proves the non-expansiveness, mean-field, residual-conversion and convergence results. We now establish the non-asymptotic convergence of the coupled SKM recursion \eqref{eq:coupled-skm}. 
\begin{theorem}[Coupled recursion Markovian convergence]\label{thm:online-main}
Let $(S_k)_{k\ge 0}$ be the policy-induced Markov chain, which is irreducible and aperiodic with stationary distribution $\mu$, let $Y_k = (S_k, R_k, S_{k+1})$ be the associated uncentered one-sample process, and let $\lambda \ge \varsigma^{-1/2}$. First, if $\alpha_k = (k+1)^{-a}$ with $\frac{4}{5} < a \le 1$, then the coupled SKM recursion \eqref{eq:coupled-skm} converges a.s.\ to a random limit $(p^\star, \bar r^\pi)$ such that $p^\star \in \mathrm{fix}(\mathcal G)$. Second, fix any $a_1 \in \bigl(\frac{4}{5}, 1\bigr)$ and set $\varepsilon_{a_1}:=(5a_1-4)/10$. There is an explicit two-phase step size schedule and constants $C^{\mathcal Z}_{\mathrm{mk}}, \kappa_{a_1} >0$ such that, for all $k \ge 1$, 
\begin{equation}
\begin{gathered}
\E\bigl[\ell^\Theta_{\mathrm C, \infty}\bigl(p_k, \mathcal G(p_k)\bigr)\bigr] \le \frac{C^{\mathcal Z}_{\mathrm{mk}}}{\mu_{\min}}(k+1)^{-1/10}\min \left\{\kappa_{a_1} \log (k+1), (k+1)^{\varepsilon_{a_1}} \right\},\\ 
\E[\lvert g_k - \bar r^\pi \rvert] \le \frac{C^{\mathcal Z}_{\mathrm {mk}}}{\lambda}(k+1)^{-1/10}\min \left\{\kappa_{a_1} \log (k+1), (k+1)^{\varepsilon_{a_1}} \right\},
\end{gathered}
\end{equation}
where $C_{\mathrm{mk}}^{\mathcal Z}$ depends on $a_1$, $\mathrm{diam}_{d_\lambda}(\mathcal Z) = \sqrt{\theta_d - \theta_1} + \lambda$, 
and on the corresponding finite-state Poisson constant for the uncentered chain $(Y_k)$. The fixed-exponent choice $a = \frac{4}{5}+\epsilon$ for  $\epsilon >0$ gives a.s.\ convergence with expected residual and gain error decaying at rate $O(k^{-1/10+\epsilon/2})$.
\end{theorem}

The $k^{-1/10}$ rate comes from the joint guarantee, and sharper rates for $g_k$ follow from standard scalar SA results, since the gain recursion does not depend on $p$ (e.g., the running reward average satisfies $\E|g_k-\bar r^\pi|=O(k^{-1/2})$). Because $C^{\mathcal Z}_{\mathrm{mk}}$ grows with $\lambda$, the smallest admissible choice $\lambda=\varsigma^{-1/2}$ is preferable. Appendix~\ref{sec:synchronous} shows that unit-step synchronous backups avoid the gain at the exact quotient-law level, a structure that projection and averaging break.

\section{Experimental results}
We present two fixed-policy evaluation experiments: five-state MDPs validating the quotient-categorical and SA results, and a continuous-state \texttt{Pendulum-v1} experiment \citep{gym} in Appendix~\ref{sec:pendulum}, illustrating neural function approximation. The experiments directly test whether the algorithms behave according to the theoretical predictions rather than compare task performance against unrelated baselines.
Both use a uniformly spaced categorical support $\Theta$ and the supremum-Cramér metric.

\noindent\textbf{Five-state MDPs.}
One five-state MDP compares exact KM, centered SKM with i.i.d.\ and Markovian samples, and the coupled recursion. The second illustrates stopped-return distributions. Appendix~\ref{app:tabular-experiment} gives the transition matrices, rewards, grids, and stepsize schedules.

Figure~\ref{fig:mrp_main} reports the residual $\ell^\Theta_{\mathrm C, \infty}(p_k,\mathcal G(p_k))$ (bounded in Theorems~\ref{thm:centered-iid}, \ref{thm:centered-markov} and~\ref{thm:online-main}) and the gain error. In panel (a), both stochastic residuals decay at roughly $k^{-1/3}$ to $k^{-2/5}$ (Appendix~\ref{app:tabular-experiment}), well ahead of the worst-case rates $k^{-1/8}$ and $k^{-3/40}$ guaranteed for these step sizes. In panel (b), estimating the gain online costs essentially nothing relative to the known-gain recursion, consistent with Theorem~\ref{thm:online-main}. Panel (c) isolates centering: with $g=0$ every sampled shift is nonnegative, the iterates approach $\mathrm{fix}(\mathcal G_0)$ where all mass sits at $\theta_d$, and the residual to $\mathcal G$ grows instead of decaying.

\begin{figure}[t]
\centering
\includegraphics[width=\textwidth]{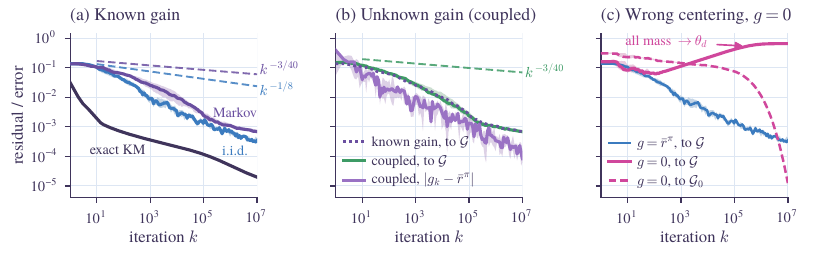}
\caption{
Five-state MDP. Residuals are $\ell^\Theta_{\mathrm C,\infty}(p_k,\mathcal G(p_k))$ unless labeled $\mathcal G_0$. Stochastic runs show medians over ten seeds with interquartile bands. \textbf{(a)} Known gain. \textbf{(b)} Unknown gain, with the known-gain Markovian residual from (a) dotted. \textbf{(c)} Centering with $g=0$, which targets $\mathcal G_0$. Dashed lines: theorem rates for the step sizes used, up to constants.
}
\label{fig:mrp_main}
\end{figure}

\begin{figure}[t]
\centering
\includegraphics[width=\textwidth]{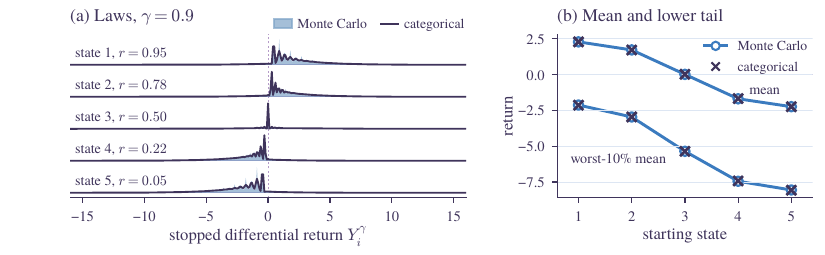}
\caption{
Stopped differential returns, $\gamma=0.9$ (mean horizon: ten rewards). \textbf{(a)} Projected Monte Carlo laws ($10^5$ returns per state, shaded blue) and categorical fixed points (dark lines), each scaled to its peak. \textbf{(b)} Means and worst-$10\%$ means per starting state.
}
\label{fig:mrp_distributions}
\end{figure}

Figure~\ref{fig:mrp_distributions} evaluates the target from Theorem~\ref{thm:stopped-main}. Persistent high- and low-reward regions produce distinct, skewed laws across starting states, and the categorical fixed points track the Monte Carlo laws and lower-tail summaries. The grid $[-20,20]$ is wide enough for clipping to be negligible, since fewer than $0.03\%$ of sampled stopped returns fall outside it.

\section{Conclusion}

We introduced a quotient-categorical formulation for average-reward distributional evaluation of differential returns, resolving the additive-constant ambiguity of the scalar bias at the level of state-indexed laws. This viewpoint leads to a projected operator that is well-defined and non-expansive in a coordinate Cram\'er metric, and to TD-style SA schemes with a.s.\ convergence and finite-iteration residual guarantees. When the gain is unknown, augmenting the categorical recursion with an online gain estimate preserves the same non-expansive structure and yields joint control of the projected residual and gain error. The projected fixed points arise as limits of geometrically stopped differential-return laws, and their means recover the scalar bias up to an explicit clipping error. Experiments on a five-state MDP agree with the theory, and a neural \texttt{Pendulum-v1} experiment in Appendix~\ref{sec:pendulum} illustrates it qualitatively. The same machinery should also extend to multivariate categorical representations by replacing the one-dimensional cumulative embedding and projection with their multivariate analogues, following the construction of \citet{kaya2026sharpfiniteiterationtheoryasynchronous}. Our theory covers finite-state fixed-policy evaluation with ergodic sampling and a fixed categorical grid. A reducible chain can leave states unvisited and have class-dependent gains, and without a single closed class in the state--atom kernel the projected fixed point need not be unique, although the Abel limit selects a canonical one. A stochastic anchored-selection theorem, nonlinear function approximation, and control remain open.

\begin{ack}
This work is supported in part by NSF DMS-2502560 and CNS-2313109.
\end{ack}

\newpage
\bibliography{refs}

\newpage
\appendix
\section*{Appendix Table of Contents}
\addcontentsline{toc}{section}{Appendix}
\markboth{Appendix}{Appendix}
\startcontents[appendix]
\printcontents[appendix]{l}{1}{\setcounter{tocdepth}{3}}

\newpage
\section{Proofs of Section~\ref{sec:quotient}}\label{app:A}

\subsection{Proof of Proposition~\ref{prop:T}}
\begin{proof}[Proof of Proposition~\ref{prop:T}]
Fix $\eta \in \mathcal F^\X$, $c \in \R$, and $i \in \X$. For each $j$, convolution commutes with common translation:
\begin{equation}
\nu_{ij}\,\ast\,\bigl((\tau_c)_\# \eta_j \bigr) = (\tau_c)_\#(\nu_{ij}\,\ast\,\eta_j).
\end{equation}
Since pushforward by the same translation is linear on mixtures,
\begin{equation}
\begin{aligned}
\bigl(\mathcal T\bigl((\tau_c)_\#\eta\bigr)\bigr)_i &= \sum_{j \in \X}P^\pi_{ij} \,\nu_{ij}\,\ast\,\bigl((\tau_c)_\#\eta_j\bigr)\\
&=\sum_{j\in\X} P^\pi_{ij}\,(\tau_c)_\# (\nu_{ij}\,\ast\,\eta_j)\\
&=(\tau_c)_\# \sum_{j \in \X} P^\pi_{ij} (\nu_{ij}\,\ast\,\eta_j) = \bigl((\tau_c)_\# \mathcal T \eta \bigr)_i.
\end{aligned}
\end{equation}
This proves \eqref{eq:commute}. Therefore, if $\eta \approx \zeta$, then $\mathcal T \eta\approx \mathcal T \zeta$, so $\overline{\mathcal T}([\eta]) := [\mathcal T \eta]$ is independent of the representative.
\end{proof}

\subsection{Proof of Proposition~\ref{prop:tg}}
\begin{proof}[Proof of Proposition~\ref{prop:tg}]
Fix $i \in \X$. Since $R_{ij} - g' = (R_{ij} -g) + (g - g')$, we have
\begin{equation}
\nu^{(g')}_{ij} = (\tau_{g-g'})_\# \nu^{(g)}_{ij}.
\end{equation}
Using the shift-equivariance of translation and convolution:
\begin{equation}
\nu^{(g')}_{ij}\,\ast\,\eta_j = \bigl((\tau_{g-g'})_\#\nu^{(g)}_{ij} \bigr)\,\ast\,\eta_j = (\tau_{g-g'})_\#(\nu^{(g)}_{ij}\,\ast\,\eta_j).
\end{equation}
Summing over $j$ therefore yields
\begin{equation}
(\mathcal T_{g'}\eta)_i = (\tau_{g-g'})_\#(\mathcal T_g \eta)_i.
\end{equation}
Hence, $\mathcal T_{g'}\eta$ and $\mathcal T_g \eta$ differ by a common translation and thus determine the same equivalence class.
\end{proof}

\subsection{Proof of Lemma~\ref{lem:proj-shift-equiv}}
\begin{proof}[Proof of Lemma~\ref{lem:proj-shift-equiv}]
Write $\Pi^\Theta_{\mathrm C} \eta = \sum_{k=1}^d p_k \delta_{\theta_k}$. The coefficients $p_k$ are determined by the usual linear interpolation weights that split mass between neighboring support points in $\Theta$, together with edge clipping outside the interval $[\theta_1, \theta_d]$. Replacing $\eta$ by $(\tau_c)_\#\eta$ and the support $\Theta$ by $\Theta_c$ translates every relevant location by the same scalar $c$, so the interpolation weights are unchanged. Consequently, the projected coefficients are the same, and the projected law is translated by $c$.
\end{proof}

\subsection{Proof of Proposition~\ref{prop:cramer}}
\begin{proof}[Proof of Proposition~\ref{prop:cramer}]
Fix a state $i \in \X$ and a common shift $c \in \R$. Since both $\eta^{p, c}_i$ and $\eta^{q, c}_i$ are supported on the same ordered grid $\Theta_c$, the difference of their cdfs is constant on each interval $[\theta_k +c, \theta_{k+1} +c)$, where it equals $\sum_{j=1}^k (p_{i, j} - q_{i,j})$, and it vanishes outside $[\theta_1+c,\theta_d+c)$ because both cdfs equal $0$ below $\theta_1+c$ and $1$ from $\theta_d+c$ on. Therefore
\begin{equation}
\begin{aligned}
\ell_{\mathrm C}(\eta^{p, c}_i, \eta^{q, c}_i)^2 &= \sum_{k=1}^{d-1}(\theta_{k+1}-\theta_k)\Bigl(\sum_{j=1}^k(p_{i, j} - q_{i,j}) \Bigr)^2 \\
&=\varsigma \sum_{k=1}^{d-1}\bigl(F_{p_i}(\theta_k) - F_{q_i}(\theta_k) \bigr)^2 = \ell^\Theta_{\mathrm C}(p_i, q_i)^2.
\end{aligned}
\end{equation}
Taking square roots proves the statewise identity, and the block-supremum identity follows by taking the maximum over the states.
\end{proof}

\subsection{Proof of Theorem~\ref{thm:gg-nonexp}}
We first state two lemmas.
\begin{lemma}\label{lem:applemma1}
For every $g \in \R$ and all $\eta, \zeta \in \mathcal F^\X$,
\begin{equation}
\ell_{\mathrm C, \infty}(\mathcal T_g\eta, \mathcal T_g\zeta) \le \ell_{\mathrm C, \infty}(\eta, \zeta).
\end{equation}
\end{lemma}
\begin{proof}
Fix $i \in \X$. The output law at state $i$ is a mixture of shifted successor laws, so
\begin{equation}
F_{(\mathcal T_g \eta)_i}(x) = \sum_{j\in\X} P^\pi_{ij}\E\bigl[F_{\eta_j}(x-
R_{ij}+g) \bigr],
\end{equation}
and similarly for $\zeta$. Jensen's inequality then gives
\begin{equation}
\ell_{\mathrm C}\bigl((\mathcal T_g\eta)_i, (\mathcal T_g\zeta)_i \bigr)^2 \le \sum_{j \in \X}P^\pi_{ij}\E\Bigl[\int_\R\bigl(F_{\eta_j}(x-R_{ij}+g) - F_{\zeta_j}(x-R_{ij}+g)\bigr)^2 \mathrm dx \Bigr].
\end{equation}
A change of variables gets rid of the translation, yielding
\begin{equation}
\ell_{\mathrm C}\bigl((\mathcal T_g\eta)_i, (\mathcal T_g\zeta)_i \bigr)^2 \le \sum_{j \in \X} P^\pi_{ij} \ell_{\mathrm C}(\eta_j, \zeta_j)^2 \le \ell_{\mathrm C, \infty}(\eta, \zeta)^2.
\end{equation}
Taking the maximum of the left-hand side over states and taking square roots proves the claim.
\end{proof}

\begin{lemma}\label{lem:applemma2}
For each $c \in \R$, the scalar categorical projection $\Pi^{\Theta_c}_\mathrm C$ is non-expansive in the Cram\'er metric. Consequently, its statewise product is non-expansive in $\ell_{\mathrm C, \infty}$.
\end{lemma}
\begin{proof}
This is the standard scalar categorical Cram\'er projection property, see \citet{rowland2018analysis}. Applying it statewise and taking the supremum over the blocks yield the claim.
\end{proof}
\begin{proof}[Proof of Theorem~\ref{thm:gg-nonexp}]
Fix $g \in \R$. To prove well-definedness, fix $p \in \Delta^\X_d$, and two representatives $\eta^{p, c}$ and $\eta^{p, c'}$. By the argument of Proposition~\ref{prop:T}, which applies verbatim to $\mathcal T_g$,
\begin{equation}
\mathcal T_g \eta^{p, c'} = (\tau_{c'-c})_\# \mathcal T_g \eta^{p, c}.
\end{equation}
Since $\eta^{p,c}=(\tau_c)_\#\eta^{p,0}$, the preceding display with $(c',c)$ replaced by $(c,0)$ gives $\mathcal T_g\eta^{p,c}=(\tau_c)_\#\mathcal T_g\eta^{p,0}$, and Lemma~\ref{lem:proj-shift-equiv} applied statewise yields $\Pi^{\Theta_c}_{\mathrm C}(\mathcal T_g\eta^{p,c})=(\tau_c)_\#\Pi^{\Theta}_{\mathrm C}(\mathcal T_g\eta^{p,0})$. The same holds with $c'$ in place of $c$, so both representatives yield the same probability coefficients.

For non-expansiveness of $\mathcal G_g$ fix the representative shift $c = 0$. Then
\begin{equation}
\eta^{\mathcal G_g(p), 0} = \Pi^\Theta_\mathrm C(\mathcal T_g \eta^{p, 0}), \qquad \eta^{\mathcal G_g(q), 0} =\Pi^\Theta_\mathrm C(\mathcal T_g \eta^{q, 0}).
\end{equation}
Using Proposition~\ref{prop:cramer} and Lemmas \ref{lem:applemma1} and \ref{lem:applemma2},
\begin{equation}
\begin{aligned}
\ell^\Theta_{\mathrm C, \infty}\bigl(\mathcal G_g(p), \mathcal G_g(q)\bigr) &= \ell_{\mathrm C, \infty}(\eta^{\mathcal G_g(p), 0}, \eta^{\mathcal G_g(q), 0}) \\
&\le \ell_{\mathrm C, \infty}(\mathcal T_g \eta^{p,0}, \mathcal T_g \eta^{q, 0}) \\
&\le \ell_{\mathrm C, \infty}(\eta^{p, 0}, \eta^{q, 0}) = \ell^\Theta_{\mathrm C, \infty}\bigl(p, q).
\end{aligned}
\end{equation}
\end{proof}

\subsection{Proof of Corollary~\ref{corol:fixed-points}}
\begin{proof}[Proof of Corollary~\ref{corol:fixed-points}]
$\Delta^\X_d$ is a nonempty compact convex subset of a finite-dimensional Euclidean space. By Theorem~\ref{thm:gg-nonexp}, $\mathcal G_g$ is $1$-Lipschitz in the cumulative-coordinate metric, which on this finite-dimensional set is equivalent to the Euclidean metric, so $\mathcal G_g$ is continuous. Brouwer's fixed-point theorem \citep{aliprantis2006infinite} therefore guarantees the existence of fixed points.
\end{proof}

\newpage
\section{Interpretation and approximation results}\label{app:meaning}

This section collects the formal statements behind the interpretation summarized in Section~\ref{sec:meaning}. Let $C_\Theta(x)$ denote a draw from the categorical projection $\Pi^\Theta_{\mathrm C}\delta_x$. Thus, inside the grid it takes one of the two neighboring atoms with the interpolation probabilities in \eqref{eq:categorical-projection}, and outside the grid it equals the nearest endpoint.

\subsection{Formal statements}
As in Section~\ref{sec:meaning}, assume without loss of generality that $0\in\Theta$, let $p^0$ put unit mass at $0$ in every state, and define $p^{n+1}=\mathcal G(p^n)$. Equivalently, if $Z_i^{(0)}=0$, the recursively projected $n$-step return obeys
\begin{equation}\label{eq:recursive-projected-return}
Z_i^{(n+1)}=C_\Theta\bigl(B_{iJ}+Z_J^{(n)}\bigr),\qquad J\sim P^\pi_{i\cdot},
\end{equation}
with a fresh successor, a reward drawn conditionally on that transition, and a fresh projection draw at each step.

\begin{theorem}[Geometric target and Abel endpoint]\label{thm:geometric-target}
Let $\gamma\in(0,1)$ and define
\begin{equation}\label{eq:stopped-operators}
(\mathcal T_\gamma\eta)_i:=\sum_jP^\pi_{ij}\,\nu_{ij}\ast\bigl((1-\gamma)\delta_0+\gamma\eta_j\bigr),
\qquad
\mathcal G_\gamma(p):=(1-\gamma)\mathcal G(p^0)+\gamma\mathcal G(p).
\end{equation}
Then $\mathcal T_\gamma$ and $\mathcal G_\gamma$ are $\gamma$-contractive in $\ell_{\mathrm C,\infty}$ and $\ell^\Theta_{\mathrm C,\infty}$, respectively. The unique fixed point of $\mathcal T_\gamma$ is the law $\eta^\gamma$ of the stopped return in \eqref{eq:geometric-return}, while the unique fixed point of $\mathcal G_\gamma$ is
\begin{equation}\label{eq:geometric-categorical-fixed-point}
p_\gamma=(1-\gamma)\sum_{n=1}^\infty\gamma^{n-1}\mathcal G^n(p^0).
\end{equation}
Moreover, with $h:=r^\pi-\bar r^\pi\mathbf1$,
\begin{equation}\label{eq:geometric-mean}
\E[Y^\gamma]=(I-\gamma P^\pi)^{-1}h=\sum_{t=0}^\infty\gamma^t(P^\pi)^th\longrightarrow v^\circ
\quad\text{as }\gamma\uparrow1,
\end{equation}
and
\begin{equation}\label{eq:abel-endpoint}
p_\gamma\longrightarrow p^\dagger:=\lim_{N\to\infty}\frac1N\sum_{n=1}^N\mathcal G^n(p^0)\in\mathrm{fix}(\mathcal G).
\end{equation}
Finally, fixed points of $\mathcal G$ are in bijection with invariant distributions of the finite state--atom kernel in \eqref{eq:state-atom-kernel}. In particular, the fixed point is unique exactly when that kernel has one closed communicating class.
\end{theorem}

Here, $p^0$ represents zero terminal continuation. After each centered reward, the trajectory stops with probability $1-\gamma$ and continues with probability $\gamma$, so $\E[N_\gamma]=1/(1-\gamma)$. This target can be sampled from the same one-step transitions used by TD by drawing an independent stopping decision after each observed reward, without requiring a model of the MDP. As $\gamma\uparrow1$, $\mathcal G_\gamma(p)\to\mathcal G(p)$ for every $p$, while the mean horizon diverges and the Abel limit above selects an undiscounted fixed point.

To see why the stopped differential-return laws are useful, and how they can distinguish between policies, consider two one-state reward processes. The first always gives reward $1/2$. In the second, two actions with rewards $0$ and $1$ are selected independently at each step with equal probability. Both have gain $1/2$ and stationary-mean-zero bias $0$. After centering, however, the first stopped return is identically zero, whereas the second sums independent increments in $\{-1/2,1/2\}$ over the geometric horizon and has variance $1/[4(1-\gamma)]$. Thus the stopped law records transient reward risk that is invisible to both the gain and the scalar bias.

For the remaining statements, take the uniform symmetric grid $\Theta=\{-M,-M+\varsigma,\ldots,M\}$ and recall that $|B_{ij}|\le1$, since rewards lie in $[0,1]$.
\begin{proposition}[Stopped-return approximation]\label{prop:stopped-approximation}
For every $\gamma\in(0,1)$,
\begin{equation}\label{eq:contraction-perturbation}
\ell_{\mathrm C,\infty}(\eta^{p_\gamma,0},\eta^\gamma)
\le \frac{1}{1-\gamma}\left(\frac{\varsigma}{4}+\frac{\gamma^{\lfloor M\rfloor}}{1-\gamma}\right)^{1/2}.
\end{equation}
Corresponding $W_1$ and CVaR bounds are given in the proof below.
\end{proposition}

For a fixed $\gamma$ and desired accuracy, this bound gives a direct grid-selection rule: increase $M$ until the geometric clipping term is negligible, then increase $d$ to reduce the spacing $\varsigma=2M/(d-1)$. Values of $\gamma$ closer to one require wider supports because they produce longer return horizons.

Recall the means $m(p)$ and the clipping defect $\beta(p)$ from \eqref{eq:clipping-defect}, and let $Z_\pi:=(I-P^\pi+\mathbf1\mu^\top)^{-1}$ and $V_\pi:=\inf_c\|v^\pi-c\mathbf1\|_\infty$.

\begin{proposition}[Poisson consistency and grid error]\label{prop:poisson-consistency}
Every $p^\star\in\mathrm{fix}(\mathcal G)$ satisfies
\begin{equation}\label{eq:perturbed-poisson}
(I-P^\pi)m(p^\star)=r^\pi-\bar r^\pi\mathbf1+\beta(p^\star),
\end{equation}
and therefore
\begin{equation}\label{eq:fixed-mean-error}
\inf_c\|m(p^\star)-v^\pi-c\mathbf1\|_\infty
\le \|Z_\pi\|_\infty\|\beta(p^\star)\|_\infty.
\end{equation}
Writing $C_\pi:=2(V_\pi+\|Z_\pi\|_\infty)+1$, the clipping defect obeys
\begin{equation}\label{eq:uniform-clipping-bound}
\|\beta(p^\star)\|_\infty\le\frac{C_\pi+\varsigma}{2M\mu_{\min}}
=\frac{C_\pi}{2M\mu_{\min}}+\frac{1}{(d-1)\mu_{\min}}.
\end{equation}
More generally, every $p\in\Delta_d^\X$ satisfies
\begin{equation}\label{eq:finite-iterate-bias}
\inf_c\|m(p)-v^\pi-c\mathbf1\|_\infty
\le\|Z_\pi\|_\infty\left(\sqrt{2M}\,\ell^\Theta_{\mathrm C,\infty}(p,\mathcal Gp)+\|\beta(p)\|_\infty\right).
\end{equation}
Finally, the clipping defect of any $p\in\Delta_d^\X$ is controlled by its residual $\ell:=\ell^\Theta_{\mathrm C,\infty}(p,\mathcal Gp)$:
\begin{equation}\label{eq:iterate-clipping-bound}
\|\beta(p)\|_\infty\le\frac{C_\pi+\varsigma+2\bigl(\|Z_\pi\|_\infty\sqrt{2M}+\sqrt{2/3}\,M^{3/2}\bigr)\ell}{2M\mu_{\min}},
\end{equation}
and therefore
\begin{equation}\label{eq:iterate-bias-accuracy}
\inf_c\|m(p)-v^\pi-c\mathbf1\|_\infty\le\|Z_\pi\|_\infty\Bigl(\frac{C_\pi+\varsigma}{2M\mu_{\min}}+K_M\,\ell\Bigr),
\end{equation}
where $K_M:=\sqrt{2M}+\bigl(\sqrt{2M/3}+\|Z_\pi\|_\infty\sqrt{2/M}\bigr)/\mu_{\min}$.
Thus fixed-point mean error vanishes as $M,d\to\infty$, and so does the mean error of any iterate whose residual is $o(M^{-1/2})$.
\end{proposition}

In particular, since $\varsigma/(2M)=1/(d-1)\le2/d$ and, for $M\ge1$, $K_M\le\sqrt M\bigl(\sqrt2+(\sqrt{2/3}+\sqrt2\,\|Z_\pi\|_\infty)/\mu_{\min}\bigr)$, Proposition~\ref{prop:mean-main} holds with
\begin{equation}\label{eq:C1C2}
C_1:=\frac{\|Z_\pi\|_\infty}{\mu_{\min}}\max\Bigl\{\frac{C_\pi}{2},\,2\Bigr\},\qquad
C_2:=\|Z_\pi\|_\infty\Bigl(\sqrt2+\frac{\sqrt{2/3}+\sqrt2\,\|Z_\pi\|_\infty}{\mu_{\min}}\Bigr).
\end{equation}

\subsection{Recursively projected returns}
Let $p^0$ put unit mass at $0$ in every state. We prove by induction that the recursion \eqref{eq:recursive-projected-return} has law $\eta_i^{\mathcal G^n(p^0),0}$ at depth $n$. The claim is immediate for $n=0$. If it holds at depth $n$, then conditioning on $J=j$ and using the conditional independence of the one-step reward and the fresh continuation return gives
\begin{equation}
\mathrm{Law}(B_{ij}+Z_j^{(n)})=\nu_{ij}\ast\eta_j^{\mathcal G^n(p^0),0}.
\end{equation}
The law of $C_\Theta(B_{ij}+Z_j^{(n)})$ is the categorical projection of this convolution. Averaging over $j$ and using linearity of categorical projection therefore gives exactly the $i$th block of $\mathcal G^{n+1}(p^0)$.

\subsection{Proof of Theorem~\ref{thm:geometric-target}}
\begin{proof}
Conditioning on the first transition and on whether the process stops after its first centered reward gives
\begin{equation}
\eta_i^\gamma=\sum_jP^\pi_{ij}\,\nu_{ij}\ast\bigl((1-\gamma)\delta_0+\gamma\eta_j^\gamma\bigr),
\end{equation}
so $\eta^\gamma$ is a fixed point of $\mathcal T_\gamma$. For any $\eta,\zeta\in\mathcal F^\X$, the cdf difference of the continuation mixtures is $\gamma(F_{\eta_j}-F_{\zeta_j})$. The same Jensen and translation argument as in Lemma~\ref{lem:applemma1} therefore yields
\begin{equation}
\ell_{\mathrm C,\infty}(\mathcal T_\gamma\eta,\mathcal T_\gamma\zeta)\le\gamma\ell_{\mathrm C,\infty}(\eta,\zeta).
\end{equation}
Existence is supplied by the stopped-return law above. If $\zeta=\mathcal T_\gamma\zeta$, then $\ell_{\mathrm C,\infty}(\zeta,\eta^\gamma)=\ell_{\mathrm C,\infty}(\mathcal T_\gamma\zeta,\mathcal T_\gamma\eta^\gamma)\le\gamma\,\ell_{\mathrm C,\infty}(\zeta,\eta^\gamma)$, so $\ell_{\mathrm C,\infty}(\zeta,\eta^\gamma)=0$ because $\gamma<1$. Hence $\eta^\gamma$ is unique.

Categorical projection is linear on mixtures, so applying it to the stopped backup gives
\begin{equation}
\mathcal G_\gamma(p)=(1-\gamma)\mathcal G(p^0)+\gamma\mathcal G(p).
\end{equation}
Theorem~\ref{thm:gg-nonexp} hence implies
\begin{equation}
\ell^\Theta_{\mathrm C,\infty}(\mathcal G_\gamma p,\mathcal G_\gamma q)\le\gamma\ell^\Theta_{\mathrm C,\infty}(p,q).
\end{equation}
Iterating the affine map gives the unique fixed point
\begin{equation}
p_\gamma=(1-\gamma)\sum_{n=1}^\infty\gamma^{n-1}\mathcal G^n(p^0).
\end{equation}
The term $\mathcal G^n(p^0)$ is the $n$-step recursively projected return proved above, and its coefficient $(1-\gamma)\gamma^{n-1}$ is precisely $\Prob(N_\gamma=n)$.

Writing $h:=r^\pi-\bar r^\pi\mathbf1$, taking expectations in the stopped Bellman equation gives
\begin{equation}
y^\gamma=h+\gamma Py^\gamma,\qquad y^\gamma:=\E[Y^\gamma],
\end{equation}
and therefore $y^\gamma=(I-\gamma P^\pi)^{-1}h=\sum_{t\ge0}\gamma^t(P^\pi)^th$. Since $\mu^\top h=0$, we also have $\mu^\top y^\gamma=0$. Under irreducibility and aperiodicity, $(P^\pi)^th$ is summable and the resolvent converges to the stationary-mean-zero Poisson solution $v^\circ$ as $\gamma\uparrow1$.

To establish the endpoint statement, note that $\mathcal G$ is linear on the coefficient space, although its matrix is not stochastic. By the intertwining relation \eqref{eq:state-atom-intertwining} proved below, $w=D_\mu p$ is mapped to $wQ^n=D_\mu\mathcal G^n(p)$ for the stochastic kernel $Q$. Consequently, the powers of $\mathcal G$ are bounded on $\Delta_d^\X$, its Ces\`aro averages converge by the finite-dimensional mean ergodic theorem for $Q$, and their limit is fixed by $\mathcal G$. The finite-dimensional Abel--Ces\`aro theorem \citep{hardy1949divergent} gives
\begin{equation}
\lim_{\gamma\uparrow1}(1-\gamma)\sum_{n=1}^\infty\gamma^{n-1}\mathcal G^n(p^0)=\lim_{N\to\infty}\frac1N\sum_{n=1}^N\mathcal G^n(p^0)=p^\dagger\in\mathrm{fix}(\mathcal G).
\end{equation}

For completeness, we now give the claimed finite Markov representation. For $1\le l,k\le d$, define
\begin{equation}\label{eq:state-atom-kernel-components}
K_{ij}(l,k):=\Prob\bigl(C_\Theta(B_{ij}+\theta_l)=\theta_k\bigr),\qquad
P^\leftarrow_{ji}:=\frac{\mu_iP^\pi_{ij}}{\mu_j},
\end{equation}
and define a kernel $Q$ on $\X\times\{1,\ldots,d\}$ by
\begin{equation}\label{eq:state-atom-kernel}
Q((j,l),(i,k)):=P^\leftarrow_{ji}K_{ij}(l,k).
\end{equation}
Both $K_{ij}$ and $P^\leftarrow$ are stochastic, hence so is $Q$. If $w_{i,k}:=\mu_ip_{i,k}$, then direct substitution gives
\begin{equation}\label{eq:state-atom-intertwining}
(wQ)_{i,k}=\mu_i(\mathcal Gp)_{i,k}.
\end{equation}
Every invariant distribution of $Q$ has state marginal $\mu$, because its state marginal is invariant for the irreducible time-reversed chain $P^\leftarrow$. Thus \eqref{eq:state-atom-intertwining} is a bijection between invariant distributions of $Q$ and fixed points of $\mathcal G$. A finite Markov kernel has a unique invariant distribution exactly when it has one closed communicating class, which proves the final claim.
\end{proof}

\subsection{Proof of Proposition~\ref{prop:stopped-approximation}}
\begin{proof}
Let $\Pi$ denote statewise categorical projection onto $[-M,M]$. Since $\eta^{p_\gamma,0}=\Pi\mathcal T_\gamma\eta^{p_\gamma,0}$ and $\eta^\gamma=\mathcal T_\gamma\eta^\gamma$, contraction and non-expansiveness of projection imply
\begin{equation}
\begin{aligned}
\ell_{\mathrm C,\infty}(\eta^{p_\gamma,0},\eta^\gamma)
&\le\ell_{\mathrm C,\infty}(\Pi\mathcal T_\gamma\eta^{p_\gamma,0},\Pi\mathcal T_\gamma\eta^\gamma)+\ell_{\mathrm C,\infty}(\Pi\eta^\gamma,\eta^\gamma)\\
&\le\gamma\ell_{\mathrm C,\infty}(\eta^{p_\gamma,0},\eta^\gamma)+\ell_{\mathrm C,\infty}(\Pi\eta^\gamma,\eta^\gamma).
\end{aligned}
\end{equation}
Rearranging gives the standard contraction-perturbation bound
\begin{equation}\label{eq:projection-perturbation-general}
\ell_{\mathrm C,\infty}(\eta^{p_\gamma,0},\eta^\gamma)\le\frac{\ell_{\mathrm C,\infty}(\Pi\eta^\gamma,\eta^\gamma)}{1-\gamma}.
\end{equation}

For a point $x$ inside one grid cell, a direct integration of the two cdfs in \eqref{eq:categorical-projection} gives
\begin{equation}
\ell_{\mathrm C}(\Pi\delta_x,\delta_x)^2=\frac{(x-\theta_k)(\theta_{k+1}-x)}{\varsigma}\le\frac{\varsigma}{4}.
\end{equation}
For $|x|>M$, endpoint clipping gives $\ell_{\mathrm C}(\Pi\delta_x,\delta_x)^2=|x|-M$. Convexity under mixtures therefore yields
\begin{equation}\label{eq:projection-pointwise-bound}
\ell_{\mathrm C}(\Pi\xi,\xi)^2\le\frac{\varsigma}{4}+\E_\xi[(|X|-M)_+].
\end{equation}
Since $|Y_i^\gamma|\le bN_\gamma$,
\begin{equation}\label{eq:geometric-overflow}
\Prob(|Y_i^\gamma|>M)\le\gamma^{\lfloor M\rfloor},\qquad
\E[(|Y_i^\gamma|-M)_+]\le\frac{\gamma^{\lfloor M\rfloor}}{1-\gamma}.
\end{equation}
Combining \eqref{eq:projection-perturbation-general}--\eqref{eq:geometric-overflow} proves \eqref{eq:contraction-perturbation}.

We finally record the promised risk-metric conversion. If $\xi$ and $\zeta$ are supported on $[-M,M]$, Cauchy--Schwarz applied to their cdfs gives
\begin{equation}\label{eq:wasserstein-cramer}
W_1(\xi,\zeta)=\int_{-M}^M|F_\xi-F_\zeta|\le\sqrt{2M}\,\ell_{\mathrm C}(\xi,\zeta).
\end{equation}
Let $e_\gamma:=\gamma^{\lfloor M\rfloor}/(1-\gamma)$ and let $\bar\eta_i^\gamma$ be the law of $[Y_i^\gamma]_{-M}^M$. Then $W_1(\eta_i^\gamma,\bar\eta_i^\gamma)\le e_\gamma$ and $\ell_{\mathrm C}(\eta_i^\gamma,\bar\eta_i^\gamma)\le\sqrt{e_\gamma}$. Hence, with $D_\gamma$ denoting the right-hand side of \eqref{eq:contraction-perturbation},
\begin{equation}\label{eq:wasserstein-stopped-bound}
W_1(\eta_i^{p_\gamma,0},\eta_i^\gamma)\le e_\gamma+\sqrt{2M}(D_\gamma+\sqrt{e_\gamma}).
\end{equation}
For $\operatorname{CVaR}_\alpha(X):=\inf_t\{t+(1-\alpha)^{-1}\E[(X-t)_+]\}$, the map from a law to CVaR is $(1-\alpha)^{-1}$-Lipschitz in $W_1$. Dividing \eqref{eq:wasserstein-stopped-bound} by $1-\alpha$ therefore gives an explicit CVaR approximation bound. The corresponding downside-CVaR bound follows by applying the same statement to $-X$.

Consequently, the stopped laws support fixed-policy comparisons through quantiles, downside probabilities, and CVaR at the chosen expected horizon.
\end{proof}

\subsection{Proof of Proposition~\ref{prop:poisson-consistency}}
\begin{proof}
Categorical interpolation preserves the input mean inside $[-M,M]$, while endpoint projection clips it. Consequently,
\begin{equation}\label{eq:mean-of-G}
m(\mathcal Gp)=r^\pi-\bar r^\pi\mathbf1+Pm(p)+\beta(p).
\end{equation}
At a fixed point, this identity becomes \eqref{eq:perturbed-poisson}. Left multiplication by $\mu^\top$ gives $\mu^\top\beta(p^\star)=0$. Since $Z_\pi(I-P^\pi)=I-\mathbf1\mu^\top$, comparison with the Poisson equation yields
\begin{equation}
m(p^\star)-v^\pi=Z_\pi\beta(p^\star)+c\mathbf1
\end{equation}
for some $c$, proving \eqref{eq:fixed-mean-error}.

For a general $p$, set $e(p):=m(p)-m(\mathcal Gp)$. Equation~\eqref{eq:mean-of-G} gives
\begin{equation}
(I-P^\pi)m(p)=r^\pi-\bar r^\pi\mathbf1+\beta(p)+e(p).
\end{equation}
The difference of the means of two laws on $[-M,M]$ is their integrated cdf difference, so Cauchy--Schwarz gives
\begin{equation}
\|e(p)\|_\infty\le\sqrt{2M}\,\ell^\Theta_{\mathrm C,\infty}(p,\mathcal Gp).
\end{equation}
Applying $Z_\pi$ proves \eqref{eq:finite-iterate-bias}.

To prove the uniform clipping estimate, first note that $|\beta_i(p^\star)|\le1$, since shifting a point in $[-M,M]$ by $B_{ij}$ can move it beyond an endpoint by at most $|B_{ij}|$. Let $I\sim\mu$, $J\sim P^\pi_{I\cdot}$, $X\sim\eta^{p_J^\star,0}$, $B\sim\nu_{IJ}$, and $C=C_\Theta(X+B)$, with the continuation and projection draws conditionally independent. Define
\begin{equation}
D(x):=x^2-([x]_{-M}^M)^2\ge2M(|x|-M)_+.
\end{equation}
For $x\in[-M,M]$, categorical interpolation satisfies
\begin{equation}
\E[C_\Theta(x)^2]=x^2+V(x),\qquad 0\le V(x)\le\varsigma
\end{equation}
when $x=X+B$: the interpolation variance is the product of the distances to the two neighboring atoms, one of which is at most $|B|$, while both are at most $\varsigma$. At a fixed point, $C$ and $X$ have the same distribution after averaging with $I\sim\mu$, which gives
\begin{equation}\label{eq:second-moment-clipping}
\E[D(X+B)]=2\E[BX]+\E[B^2]+\E[V(X+B)].
\end{equation}
The centered reward has $\E[B]=0$ under the stationary transition law. Moreover, \eqref{eq:fixed-mean-error} and the preliminary bound $\|\beta(p^\star)\|_\infty\le1$ allow a constant $c$ such that
\begin{equation}
\max_j|m_j(p^\star)-c|\le V_\pi+\|Z_\pi\|_\infty.
\end{equation}
Conditional independence gives
\begin{equation}
|\E[BX]|=|\E[B(m_J(p^\star)-c)]|\le V_\pi+\|Z_\pi\|_\infty.
\end{equation}
Thus \eqref{eq:second-moment-clipping} is at most $C_\pi+\varsigma$. On the other hand,
\begin{equation}
\mu_i|\beta_i(p^\star)|\le\E[(|X+B|-M)_+\mathbf1\{I=i\}]\le\frac{\E[D(X+B)]}{2M}.
\end{equation}
Taking the maximum over $i$, using $\mu_i\ge\mu_{\min}$ and $\varsigma=2M/(d-1)$, proves \eqref{eq:uniform-clipping-bound}.

For a general $p$, repeat this argument with $p^\star$ replaced by $p$. The preliminary bound $|\beta_i(p)|\le1$ still holds. Since $C$ has law $\sum_i\mu_i\eta^{(\mathcal Gp)_i,0}$ and $X$ has law $\sum_j\mu_j\eta^{p_j,0}$, identity \eqref{eq:second-moment-clipping} acquires the correction $\E[X^2]-\E[C^2]$ on its right-hand side. For laws $\xi,\zeta$ on $[-M,M]$, integration by parts and Cauchy--Schwarz give
\begin{equation}
\bigl|\E_\xi[X^2]-\E_\zeta[X^2]\bigr|=\Bigl|\int_{-M}^M2x\bigl(F_\zeta(x)-F_\xi(x)\bigr)\,\mathrm dx\Bigr|\le2\sqrt{2M^3/3}\;\ell_{\mathrm C}(\xi,\zeta),
\end{equation}
so the correction is at most $2\sqrt{2/3}\,M^{3/2}\ell$. In the bound on $\E[BX]$, \eqref{eq:finite-iterate-bias} and $\|\beta(p)\|_\infty\le1$ give a constant $c$ with $\max_j|m_j(p)-c|\le V_\pi+\|Z_\pi\|_\infty(1+\sqrt{2M}\,\ell)$. With these two changes, the argument above yields \eqref{eq:iterate-clipping-bound}, and substituting it into \eqref{eq:finite-iterate-bias} yields \eqref{eq:iterate-bias-accuracy}.
\end{proof}

\newpage
\section{Proofs of Section~\ref{sec:centered}}\label{app:B}
For the convenience of proofs, let us write $L_b : \Delta_d \to \Delta_d$ for the unique coefficient map satisfying
\begin{equation}
\Pi^\Theta_\mathrm C \bigl((\tau_b)_\# \eta^{p,0} \bigr) = \eta^{L_b p,0},\qquad p \in \Delta_d, b\in\R.
\end{equation}

We first fix three conventions used in all SA proofs below. First, following \citet{bravo2024stochastic} and \citet{blaser2026asymptotic}, we write the recursion as $p_k=p_{k-1}+\alpha_k\bigl(H(p_{k-1},Y_k)-p_{k-1}\bigr)$ for $k\ge1$. This is \eqref{eq:KM} started from its first iterate $p_1=H(p_0,Y_0)$, with indices shifted by one. Every step size $\alpha_k=(k+1)^{-a}$ used in the analysis therefore lies in $(0,1)$, the iterate $p_{k-1}$ is measurable with respect to $Y_0,\dots,Y_{k-1}$, and, for $k\ge2$, the shift changes the constants in the residual bounds by at most a factor of $2$. The first iterate, which uses $\alpha_0=1$, is covered by the trivial bound $\mathrm{diam}_{\ell^\Theta_{\mathrm C,\infty}}(\Delta_d^\X)$ after enlarging constants.

Second, we identify $p\in\Delta_d^\X$ with $F(p):=\bigl(F_{p_i}(\theta_k)\bigr)_{i\in\X,\,k\le d-1}$, so that $\ell^\Theta_{\mathrm C,\infty}(p,q)=\sqrt{\varsigma}\max_i\|F(p)_i-F(q)_i\|_2=:\|F(p)-F(q)\|_{\mathrm C,\infty}$, and let $\|\cdot\|_{\mathrm C,2}$ denote $\sqrt\varsigma$ times the ambient Euclidean norm, so that $\|\cdot\|_{\mathrm C,\infty}\le\|\cdot\|_{\mathrm C,2}\le\sqrt{|\X|}\,\|\cdot\|_{\mathrm C,\infty}$. The image of $\Delta_d^\X$ is $K_0^\X$ with $K_0:=\{u\in\R^{d-1}:0\le u_1\le\cdots\le u_{d-1}\le1\}$, a compact convex set.

Third, the SA results we cite are stated for maps defined on the whole space, and their proofs evaluate the mean-field map at auxiliary points such as $x_k-\overline U_k$ below, which may leave $K_0^\X$. Let $\mathrm{Proj}$ be the blockwise Euclidean projection onto $K_0^\X$. Euclidean projection onto a closed convex set is $1$-Lipschitz, so $\mathrm{Proj}$ is non-expansive in $\|\cdot\|_{\mathrm C,\infty}$. Hence $\tilde H(\cdot,y):=H(\mathrm{Proj}(\cdot),y)$ and $\tilde h:=h\circ\mathrm{Proj}$, for $h\in\{h_\rho,h_\mu\}$, are non-expansive on the whole space, take values in $K_0^\X$, coincide with $H(\cdot,y)$ and $h$ on $K_0^\X$, and satisfy $\mathrm{fix}(\tilde h)=\mathrm{fix}(h)$, since every fixed point lies in the range $K_0^\X$. The SKM iterates are convex combinations of points of $K_0^\X$, so replacing $(H,h)$ by $(\tilde H,\tilde h)$ changes neither the iterates nor the noise, and the bounded-range alternative of hypothesis (H4) in \citet{bravo2024stochastic} holds with constant $\mathrm{diam}_{\ell^\Theta_{\mathrm C,\infty}}(\Delta^\X_d)$. All SA results below are applied to these extended maps. For the coupled recursion of Appendix~\ref{app:C}, the gain coordinate is extended in the same way by clipping to $[0,1]$.

\subsection{Proof of Proposition~\ref{prop:g-block-form}}
\begin{proof}[Proof of Proposition~\ref{prop:g-block-form}]
Fix $g \in \R$ and $i \in \X$ and consider the representative shift $c=0$. By definition of $\mathcal G_g$,
\begin{equation}
\eta^{\mathcal G_g(p), 0}_i = \Pi^\Theta_\mathrm C\Bigl(\sum_{j \in \X} P^\pi_{ij}(\nu^{(g)}_{ij}\,\ast\,\eta^{p, 0}_j) \Bigr).
\end{equation}
Because scalar categorical projection is linear on mixtures,
\begin{equation}
\eta^{\mathcal G_g(p), 0}_i = \sum_{j \in \X}P^\pi_{ij}\Pi^\Theta_\mathrm C (\nu^{(g)}_{ij}\,\ast\,\eta^{p, 0}_j).
\end{equation}
Now, fix $j \in \X$ and let $B\sim \nu^{(g)}_{ij}$. Using the definition of $L_b$,
\begin{equation}
\Pi^\Theta_\mathrm C(\delta_b \,\ast\,\eta^{p, 0}_j) = \eta^{L_bp_j, 0}\qquad \text{for every } b \in \R.
\end{equation}
Taking expectation over $B$,
\begin{equation}
\Pi^\Theta_\mathrm C(\nu^{(g)}_{ij} \,\ast\,\eta^{p, 0}_j) = \E\bigl[\eta^{L_Bp_j,0} \bigr].
\end{equation}
Since each law inside the expectation is categorical on the same support $\Theta$, taking expectation simply averages the coefficient vectors, so
\begin{equation}
\E\bigl[\eta^{L_Bp_j,0} \bigr] = \eta^{\E[L_B p_j],0}.
\end{equation}
Comparing coefficients on the common support $\Theta$ shows that
\begin{equation}
\mathcal G_g(p)_i = \sum_{j \in \X} P^\pi_{ij} \E[L_B p_j], \qquad B \sim \nu^{(g)}_{ij}.
\end{equation}
\end{proof}

\subsection{Proof of Lemma~\ref{lem:one-sample-nonexp}}
\begin{proof}[Proof of Lemma~\ref{lem:one-sample-nonexp}]
Let $y=(s,b,s')$. For blocks $u \ne s$, the sampled update leaves the block unchanged. For the updated block,
\begin{equation}
\ell^\Theta_\mathrm C\bigl(H(p, y)_s, H(q, y)_s\bigr) = \ell^\Theta_\mathrm C(L_b p_{s'}, L_b q_{s'}).
\end{equation}
Using Proposition~\ref{prop:cramer} together with projection non-expansiveness and translation invariance of $\ell_\mathrm C$ gives
\begin{equation}
\ell^\Theta_\mathrm C(L_b p_{s'}, L_b q_{s'}) \le \ell^\Theta_\mathrm C(p_{s'}, q_{s'}).
\end{equation}
Taking the maximum over all state blocks proves the result.
\end{proof}

\subsection{Proof of Proposition~\ref{prop:mean-field-fix}}
\begin{proof}[Proof of Proposition~\ref{prop:mean-field-fix}]
The expectation identity follows blockwise from Proposition~\ref{prop:g-block-form}. $\mathrm{fix}(\mathcal G) \subseteq \mathrm{fix}(h_\rho)$ follows trivially from the definition of $h_\rho$. Conversely, if $h_\rho(p) = p$, then $\rho_i(\mathcal G(p)_i - p_i) = 0$ for every $i \in \X$. $\rho_{\min}>0$ implies that $\mathcal G(p)_i = p_i$ for every $i \in \X$, hence $p \in \mathrm{fix}(\mathcal G)$.

For the non-expansiveness statement, fix $i \in \X$. Since
\begin{equation}
h_\rho(p)_i = (1-\rho_i)p_i + \rho_i\mathcal G(p)_i,\qquad h_\rho(q)_i = (1-\rho_i)q_i + \rho_i \mathcal G(q)_i,
\end{equation}
we have for each $k = 1, \dots, d-1$,
\begin{equation}
F_{h_\rho(p)_i}(\theta_k) - F_{h_\rho(q)_i}(\theta_k) = (1 - \rho_i)\bigl(F_{p_i}(\theta_k) -F_{q_i}(\theta_k)\bigr) + \rho_i \bigl(F_{\mathcal G(p)_i}(\theta_k) -  F_{\mathcal G(q)_i}(\theta_k)\bigr).
\end{equation}
Through the triangle inequality with the vector of cumulative differences with $(d-1)$ elements,
\begin{equation}
\begin{aligned}
\ell^\Theta_\mathrm C\bigl(h_\rho(p)_i, h_\rho(q)_i \bigr) &\le (1-\rho_i) \ell^\Theta_\mathrm C(p_i, q_i) + \rho_i \ell^\Theta_\mathrm C\bigl(\mathcal G(p)_i, \mathcal G(q) _i\bigr) \\
&\le \ell^\Theta_{\mathrm C, \infty}(p, q),
\end{aligned}
\end{equation}
where the last inequality uses Theorem~\ref{thm:gg-nonexp}. Taking the maximum over $i \in \X$ proves non-expansiveness.
\end{proof}

\subsection{Proof of Corollary~\ref{corol:residual-conversion}}
\begin{proof}[Proof of Corollary~\ref{corol:residual-conversion}]
For each state $i \in \X$,
\begin{equation}
h_\rho(p)_i - p_i = \rho_i\bigl(\mathcal G(p)_i -p_i \bigr).
\end{equation}
Hence, for each $k = 1, \dots, d-1$,
\begin{equation}
F_{h_\rho(p)_i}(\theta_k) - F_{p_i}(\theta_k) = \rho_i\bigl(F_{\mathcal G(p)_i}(\theta_k) - F_{p_i}(\theta_k) \bigr),
\end{equation}
and therefore
\begin{equation}
\ell^\Theta_\mathrm C\bigl(p_i, h_\rho(p)_i\bigr) = \rho_i \ell^\Theta_\mathrm C\bigl(p_i, \mathcal G(p)_i\bigr).
\end{equation}
Taking the maximum over $i$ yields
\begin{equation}
\ell^\Theta_{\mathrm C, \infty}\bigl(p, h_\rho(p)\bigr) = \max_{i \in \X} \rho_i \ell^\Theta_\mathrm C\bigl(p_i, \mathcal G(p)_i\bigr),
\end{equation}
from which the stated bounds follow immediately.
\end{proof}

\subsection{Proof of Theorem~\ref{thm:centered-iid}}
We first state the following proposition.
\begin{proposition}[Two-phase i.i.d.\ residual refinement]\label{prop:centered-iid-two-phase}
Under the assumptions of Theorem~\ref{thm:centered-iid}, fix $a_1 \in \bigl(\frac23,1\bigr)$ and set
\begin{equation}
\varepsilon_{a_1}:=\frac{3a_1-2}{6},\qquad \kappa_{a_1}:=\frac{2^{\varepsilon_{a_1}}}{\log 2}.
\end{equation}
Let $T$ be the largest integer such that
\begin{equation}\label{eq:two-phase-threshold}
(k+1)^{\varepsilon_{a_1}}\le \kappa_{a_1}\log(k+1),\qquad 1\le k\le T.
\end{equation}
If the centered SKM recursion uses the two-phase step size
\begin{equation}\label{eq:two-phase-stepsize}
\alpha_k =
\begin{cases}
(k+1)^{-a_1}, & 1\le k\le T,\\
(k+1)^{-2/3}, & k>T,
\end{cases}
\end{equation}
then there exists $C_{\mathrm{iid}} >0$ such that, for all $k\ge1$,
\begin{equation}
\E\Bigl[\ell^\Theta_{\mathrm C,\infty}\bigl(p_k,h_\rho(p_k)\bigr)\Bigr]\le C_{\mathrm{iid}}(k+1)^{-1/6}\min\Bigl\{\kappa_{a_1}\log(k+1),(k+1)^{\varepsilon_{a_1}}\Bigr\}.
\end{equation}
Consequently,
\begin{equation}
\E\Bigl[\ell^\Theta_{\mathrm C,\infty}\bigl(p_k,\mathcal G(p_k)\bigr)\Bigr]\le \frac{C_{\mathrm{iid}}}{\rho_{\min}}(k+1)^{-1/6}\min\Bigl\{\kappa_{a_1}\log(k+1),(k+1)^{\varepsilon_{a_1}}\Bigr\}.
\end{equation}
\end{proposition}
\begin{proof}
Since the tail exponent in \eqref{eq:two-phase-stepsize} is the critical value $2/3$, the a.s.\ convergence assertion remains the separate $a>\frac23$ assertion in Theorem~\ref{thm:centered-iid}. The role of the two-phase schedule is to interpolate between the near-critical transient behavior of an exponent $a_1>\frac23$ and the critical last-iterate residual behavior.

We first record the constants used below. Let
\begin{equation}
D:=\sqrt{\theta_d-\theta_1},\qquad M_2:=\sqrt{|\X|}D,
\end{equation}
and define
\begin{equation}
S_{3a_1/2}:=\sum_{t=1}^{\infty}(t+1)^{-3a_1/2},\quad \omega_{a_1}:=\frac{1-2^{-a_1}}{2},\quad \nu_{a_1}:=\frac{1-2^{-a_1}}{1-a_1},\quad \eta_{a_1}:=1-(3/4)^{1-a_1}.
\end{equation}
Also let
\begin{equation}
R_{a_1}:=\frac{(4/3)^{a_1}}{1-4^{-a_1}},
\end{equation}
\begin{equation}
K_{a_1}:=\frac{D}{\sqrt{\pi\omega_{a_1}}} +\frac{2^{1+a_1/2}M_2}{\sqrt{\pi}}\left(\frac{S_{3a_1/2}}{\sqrt{\nu_{a_1}\eta_{a_1}}} +\frac{2^{1+a_1/2}R_{a_1}}{\sqrt{1-a_1}}\right),
\end{equation}
and
\begin{equation}
B_T:=\max\left\{\frac{2}{\sqrt3},\,(32/3)^{1/6}(T+2)^{\varepsilon_{a_1}}\right\}.
\end{equation}
The final constant can be chosen as
\begin{equation}\label{eq:two-phase-constant}
C_{\mathrm{iid}}:=\max\left\{K_{a_1}+6M_2,\; K_{a_1}B_T+\frac{2^{4+a_1/2}M_2\sqrt6}{3\sqrt{\pi}}+6M_2\right\}.
\end{equation}

The threshold in \eqref{eq:two-phase-threshold} is well defined. Indeed, for $x\ge2$, the function $x^{\varepsilon_{a_1}}/\log x$ decreases until $x=\exp(1/\varepsilon_{a_1})$ and then increases to infinity. Its value at $x=2$ is $\kappa_{a_1}$. Hence the set of integers satisfying \eqref{eq:two-phase-threshold} is a nonempty initial interval, and after its last element $T$ the reverse inequality holds.

We now derive the residual estimate. Work in the cumulative-coordinate embedding associated with $\ell^\Theta_{\mathrm C,\infty}$. The block-sup Cram\'er metric is the block-sup norm of Euclidean cumulative-coordinate vectors, and this block-sup norm is bounded by the ambient Euclidean norm. Under the i.i.d.\ assumptions, the noise
\begin{equation}
U_k:=H(p_{k-1},Y_k)-h_\rho(p_{k-1})
\end{equation}
is a martingale difference and satisfies $\|U_k\|_{\mathrm C,2}\le M_2$. Define the averaged noise process
\begin{equation}
\overline U_k=(1-\alpha_k)\overline U_{k-1}+\alpha_k U_k,\qquad \overline U_0=0.
\end{equation}
By orthogonality of martingale differences in the Euclidean cumulative-coordinate space and Jensen's inequality,
\begin{equation}\label{eq:avg-noise-two-phase}
\E\bigl[\|\overline U_k\|_{\mathrm C,\infty}\bigr]\le \E\bigl[\|\overline U_k\|_{\mathrm C,2}\bigr] \le M_2\sqrt{\sum_{t=1}^k\alpha_t^2\prod_{r=t+1}^k(1-\alpha_r)^2}.
\end{equation}
Let
\begin{equation}
v_k:=\sum_{t=1}^k\alpha_t^2\prod_{r=t+1}^k(1-\alpha_r)^2.
\end{equation}
We claim that $v_k\le \alpha_k$ for all $k\ge1$. Since
\begin{equation}
v_k=(1-\alpha_k)^2v_{k-1}+\alpha_k^2,
\end{equation}
it is enough by induction to verify
\begin{equation}
\alpha_k^{-1}-\alpha_{k-1}^{-1}\le1.
\end{equation}
In the first phase this follows from the mean-value theorem:
\begin{equation}
(k+1)^{a_1}-k^{a_1}\le a_1 k^{a_1-1}<1.
\end{equation}
In the second phase,
\begin{equation}
(k+1)^{2/3}-k^{2/3}\le \frac23 k^{-1/3}<1.
\end{equation}
At the transition point,
\begin{equation}
\alpha_{T+1}^{-1}-\alpha_T^{-1} =(T+2)^{2/3}-(T+1)^{a_1} \le (T+1)^{2/3}-(T+1)^{a_1}+\frac23(T+1)^{-1/3}<1,
\end{equation}
because $a_1>2/3$. Therefore \eqref{eq:avg-noise-two-phase} gives
\begin{equation}\label{eq:avg-noise-final-two-phase}
\E\bigl[\|\overline U_k\|_{\mathrm C,\infty}\bigr]\le M_2\sqrt{\alpha_k}.
\end{equation}

Following the inexact Krasnosel'ski\u{\i}--Mann estimate of \citet[Theorem~2.10]{bravo2024stochastic}, with \eqref{eq:avg-noise-final-two-phase} and $\alpha_{t-1}\le2^{a_1}\alpha_t$, which holds in both phases and at the transition, where the step size increases, yields
\begin{equation}\label{eq:two-phase-core-bound}
\E\Bigl[\ell^\Theta_{\mathrm C,\infty}\bigl(p_k,h_\rho(p_k)\bigr)\Bigr]\le \frac{D}{\sqrt{\pi\tau_k}}+\frac{2^{1+a_1/2}M_2}{\sqrt{\pi}}\sum_{t=1}^{k-1}\frac{\alpha_t^{3/2}}{\sqrt{\tau_k-\tau_t}}+6M_2\sqrt{\alpha_k},
\end{equation}
where $\tau_k:=\sum_{t=1}^k\alpha_t(1-\alpha_t)$. It remains to bound the middle sum in the two phases.

\textbf{Primary phase.}
Assume $1\le k\le T$, so $\alpha_t=(t+1)^{-a_1}$. We use
\begin{equation}\label{eq:tau-primary-two-phase}
\omega_{a_1}(k+1)^{1-a_1}\le \tau_k\le \frac{1}{1-a_1}(k+1)^{1-a_1}.
\end{equation}
The upper bound is immediate by integral comparison. For the lower bound, $1-\alpha_t\ge1-2^{-a_1}$ and the standard estimate $\tau_k\ge \tau_1 k^{1-a_1}$ imply the displayed inequality.

Split the middle sum at $\lfloor k/2\rfloor$. For $t\le\lfloor k/2\rfloor$,
\begin{equation}
\tau_k-\tau_t\ge \nu_{a_1}\eta_{a_1}(k+1)^{1-a_1},
\end{equation}
and therefore
\begin{equation}
\sum_{t=1}^{\lfloor k/2\rfloor}\frac{\alpha_t^{3/2}}{\sqrt{\tau_k-\tau_t}} \le \frac{S_{3a_1/2}}{\sqrt{\nu_{a_1}\eta_{a_1}}}(k+1)^{-(1-a_1)/2}.
\end{equation}
For $t>\lfloor k/2\rfloor$,
\begin{equation}
\alpha_t^{1/2}\le 2^{a_1/2}(k+1)^{-a_1/2}.
\end{equation}
Using the left-shifted integral comparison
\begin{equation}
\frac{1}{\sqrt{\tau_k-\tau_t}}\le \frac{1}{\alpha_{t+1}(1-\alpha_{t+1})}\int_{\tau_t}^{\tau_{t+1}}\frac{\mathrm d\tau}{\sqrt{\tau_k-\tau}},
\end{equation}
and the ratio bound
\begin{equation}
\frac{\alpha_t}{\alpha_{t+1}(1-\alpha_{t+1})}\le R_{a_1},
\end{equation}
we get
\begin{equation}
\sum_{t=\lfloor k/2\rfloor+1}^{k-1}\frac{\alpha_t^{3/2}}{\sqrt{\tau_k-\tau_t}} \le \frac{2^{1+a_1/2}R_{a_1}}{\sqrt{1-a_1}}(k+1)^{-(1-a_1)/2}.
\end{equation}
Combining these estimates with \eqref{eq:two-phase-core-bound} and \eqref{eq:tau-primary-two-phase} gives
\begin{equation}\label{eq:primary-final-two-phase}
\E\Bigl[\ell^\Theta_{\mathrm C,\infty}\bigl(p_k,h_\rho(p_k)\bigr)\Bigr]\le (K_{a_1}+6M_2)(k+1)^{-(1-a_1)/2}.
\end{equation}
Since
\begin{equation}
(k+1)^{-(1-a_1)/2}=(k+1)^{-1/6}(k+1)^{\varepsilon_{a_1}},
\end{equation}
and \eqref{eq:two-phase-threshold} makes the minimum equal to $(k+1)^{\varepsilon_{a_1}}$ in this phase, \eqref{eq:primary-final-two-phase} proves the claimed bound for $k\le T$.

\textbf{Secondary phase.}
Now let $k>T$, so the tail stepsizes satisfy $\alpha_t=(t+1)^{-2/3}$. For any $t'\ge T$,
\begin{equation}\label{eq:tau-secondary-two-phase}
\tau_k-\tau_{t'}=\sum_{t=t'+1}^k\alpha_t(1-\alpha_t) \ge \frac12\sum_{t=t'+1}^k(t+1)^{-2/3}\ge \frac32\bigl((k+2)^{1/3}-(t'+2)^{1/3}\bigr).
\end{equation}
Split the middle sum in \eqref{eq:two-phase-core-bound} at $T$.

For the tail part $t>T$, we have $\alpha_t^{3/2}=(t+1)^{-1}\le\frac43(t+2)^{-1}$. By \eqref{eq:tau-secondary-two-phase} and the shift $(k,t)\mapsto(k+1,t+1)$, the tail part is at most $\frac43$ times the following sum with $k+1$ replaced by $k+2$, which we bound for a generic $k$:
\begin{equation}
\sum_{t=T+1}^{k-1}\frac{(t+1)^{-1}}{\sqrt{\frac32\bigl((k+1)^{1/3}-(t+1)^{1/3}\bigr)}} \le \int_T^k\frac{(x+1)^{-1}}{\sqrt{\frac32((k+1)^{1/3}-(x+1)^{1/3})}}\,\mathrm dx.
\end{equation}
With $u=(x+1)^{1/3}$ and $c=(k+1)^{1/3}$, this integral equals
\begin{equation}
\sqrt6\int_{(T+1)^{1/3}}^{(k+1)^{1/3}}\frac{\mathrm d u}{u\sqrt{(k+1)^{1/3}-u}}
\end{equation}
and is bounded by
\begin{equation}
\sqrt6\,(k+1)^{-1/6}\log(k+1).
\end{equation}
To see the last step, write $x=((k+1)/(T+1))^{1/3}>1$ in the logarithmic expression obtained after integration:
\begin{equation}
\frac{\sqrt{x}+\sqrt{x-1}}{\sqrt{x}-\sqrt{x-1}} =2x-1+2\sqrt{x^2-x} \le 4x-1\le 3x^3\le (T+1)x^3=k+1.
\end{equation}
Applied with $k+2$ in place of $k+1$, and using $(k+2)^{-1/6}\le(k+1)^{-1/6}$ and $\log(k+2)\le2\log(k+1)$, this bounds the tail part by
\begin{equation}\label{eq:tail-log-two-phase}
\frac{8\sqrt6}{3}\,(k+1)^{-1/6}\log(k+1).
\end{equation}

It remains to control the finitely many primary-phase terms $t\le T$. Evaluating the first two terms in \eqref{eq:two-phase-core-bound} at $T+1$ and using the primary-phase estimate gives
\begin{equation}
\frac{D}{\sqrt{\pi\tau_{T+1}}}+\frac{2^{1+a_1/2}M_2}{\sqrt{\pi}}\sum_{t=1}^{T}\frac{\alpha_t^{3/2}}{\sqrt{\tau_{T+1}-\tau_t}} \le K_{a_1}(T+2)^{-(1-a_1)/2}.
\end{equation}
If $k+2\ge 8(T+3)$, then \eqref{eq:tau-secondary-two-phase} with $t'=T+1$ gives
\begin{equation}
\tau_k-\tau_{T+1}\ge \frac34(k+2)^{1/3}\ge\frac34(k+1)^{1/3},
\end{equation}
and hence these early terms are bounded by
\begin{equation}
K_{a_1}\frac{2}{\sqrt3}(k+1)^{-1/6}.
\end{equation}
If instead $k+2<8(T+3)$, then $k+1<8(T+3)\le\frac{32}{3}(T+2)$, so
\begin{equation}
K_{a_1}(T+2)^{-(1-a_1)/2}\le K_{a_1}(32/3)^{1/6}(T+2)^{\varepsilon_{a_1}}(k+1)^{-1/6}.
\end{equation}
Thus, in both cases, the early terms are bounded by
\begin{equation}\label{eq:early-secondary-two-phase}
K_{a_1}B_T(k+1)^{-1/6}.
\end{equation}
Combining \eqref{eq:tail-log-two-phase}, \eqref{eq:early-secondary-two-phase}, and the terminal term $6M_2\sqrt{\alpha_k}=6M_2(k+1)^{-1/3}\le6M_2(k+1)^{-1/6}$, we obtain
\begin{equation}
\E\Bigl[\ell^\Theta_{\mathrm C,\infty}\bigl(p_k,h_\rho(p_k)\bigr)\Bigr] \le \left(K_{a_1}B_T+\frac{2^{4+a_1/2}M_2\sqrt6}{3\sqrt{\pi}}+6M_2\right)(k+1)^{-1/6}\log(k+1).
\end{equation}
For $k>T$, the definition of the last crossing $T$ gives
\begin{equation}
(k+1)^{\varepsilon_{a_1}}>\kappa_{a_1}\log(k+1),
\end{equation}
so the minimum in the proposition is $\kappa_{a_1}\log(k+1)$. Since $\kappa_{a_1}\ge1$, the choice \eqref{eq:two-phase-constant} proves the residual bound in the secondary phase as well. The projected-operator residual follows from Corollary~\ref{corol:residual-conversion}.
\end{proof}
\begin{remark}
As a concrete example, $a_1=3/4$ gives $\varepsilon_{a_1}=1/24$, and the two-phase residual bound becomes
\begin{equation}
\E\Bigl[\ell^\Theta_{\mathrm C,\infty}\bigl(p_k,\mathcal G(p_k)\bigr)\Bigr]\le \frac{C_{\mathrm{iid}}}{\rho_{\min}}(k+1)^{-1/6} \min\left\{\frac{2^{1/24}}{\log2}\log(k+1),(k+1)^{1/24}\right\}.
\end{equation}
\end{remark}

We are now in a position to state the proof of Theorem~\ref{thm:centered-iid}.
\begin{proof}[Proof of Theorem~\ref{thm:centered-iid}]
Under the theorem assumptions, the mean-field map of the recursion is $h_\rho$ by Proposition~\ref{prop:mean-field-fix}. By Lemma~\ref{lem:one-sample-nonexp} the sampled map is samplewise $1$-Lipschitz, and by Proposition~\ref{prop:mean-field-fix} and Corollary~\ref{corol:fixed-points}, the mean-field map is non-expansive with nonempty fixed-point set. Compactness of $\Delta^\X_d$ implies bounded iterates and bounded martingale-difference noise. The almost sure convergence result of \citet[Theorem~2.5 and Example~2.7]{bravo2024stochastic} therefore applies to $h_\rho$ for every $\frac{2}{3}< a \le 1$.

For the expected residual, \citet[Theorem~4.4]{bravo2024stochastic} gives the following last-iterate profile for $\frac{1}{2}\le a \le 1$: there exist constants $C_1(a), C_2, C_3(a), C_4 >0$, depending only on $a$, $|\X|$ (through the comparison of $\|\cdot\|_{\mathrm C,\infty}$ with $\|\cdot\|_{\mathrm C,2}$), the support diameter $\mathrm{diam}_{\ell^\Theta_{\mathrm C, \infty}}(\Delta^\X_d) = \sqrt{\theta_d - \theta_1}$, and any uniform bound on $\ell^\Theta_{\mathrm C, \infty}(H(p, y), p)$. Since both $H(p, y)$ and $p$ lie in $\Delta^\X_d$, this sampled-update term is itself bounded by the same diameter, so no additional problem-dependent constant is needed in our setting. Thus
\begin{equation}
\E\bigl[\ell^\Theta_{\mathrm C, \infty}\bigl(p_k, h_\rho(p_k)\bigr)\bigr] \le \begin{cases}
C_1(a)(k+1)^{-(a-\frac{1}{2})}, &\frac{1}{2}\le a < \frac{2}{3},\\
C_2\frac{\log(k+1)}{(k+1)^{1/6}}, &a=\frac{2}{3},\\
C_3(a)(k+1)^{-\frac{1-a}{2}}, &\frac{2}{3}<a<1,\\
C_4(\log(k+1))^{-\frac{1}{2}}, &a=1.
\end{cases}
\end{equation}
The two-phase residual assertion stated in the theorem is Proposition~\ref{prop:centered-iid-two-phase}. The projected-operator residual bound then follows from Corollary~\ref{corol:residual-conversion}.
\end{proof}

\subsection{Proof of Theorem~\ref{thm:centered-markov}}
We first state a lemma and then a proposition.
\begin{lemma}[Markovian inexact-KM reduction]\label{lem:markov-inexact-km}
Consider an SKM recursion $x_k=x_{k-1}+\alpha_k\bigl(H(x_{k-1},Y_k)-x_{k-1}\bigr)$ on a compact convex subset $K$ of a finite-dimensional space with norm $\|\cdot\|$, driven by a finite-state irreducible and aperiodic Markov chain $(Y_k)$, with samplewise non-expansive update map $H$ and non-expansive stationary mean-field map $h$, both extended to the whole space as in the conventions above. Let the step sizes be either $\alpha_k=(k+1)^{-a}$ with $\frac45<a<1$, or the matched two-phase schedule in Proposition~\ref{prop:centered-markov-two-phase}. Let
\begin{equation}
\tau_k:=\sum_{t=1}^k \alpha_t(1-\alpha_t),\qquad \sigma(0):=1,\qquad \sigma(u):=\min\{1,(\pi u)^{-1/2}\}\quad (u>0).
\end{equation}
There are constants $A_0,A_1,A_2>0$ such that the iterates satisfy
\begin{equation}\label{eq:markov-inexact-km}
\E\bigl[\|x_k-h(x_k)\|\bigr] \le A_0\sigma(\tau_k)\left(1+\sum_{t=2}^k \alpha_t\omega_{t-1}\right) +A_1\sum_{t=2}^k \alpha_t\sigma(\tau_k-\tau_t)\omega_{t-1} +A_1\omega_k,
\end{equation}
where the deterministic envelope $\omega_k$ obeys
\begin{equation}\label{eq:markov-poisson-envelope}
\omega_k\le A_2\tau_k\sqrt{\alpha_{k+1}}.
\end{equation}
\end{lemma}
\begin{proof}
Let $\xi_k:=H(x_{k-1},Y_k)-h(x_{k-1})$, let $\overline U_k:=(1-\alpha_k)\overline U_{k-1}+\alpha_k\xi_k$ with $\overline U_0:=0$, and set $\omega_k:=\E\|\overline U_k\|$. As in the proof of \citet[Theorem~3.1]{blaser2026asymptotic}, $z_k:=x_k-\overline U_k$ satisfies
\begin{equation}
z_k=(1-\alpha_k)z_{k-1}+\alpha_k\Bigl(h(z_{k-1})+\bigl(h(x_{k-1})-h(z_{k-1})\bigr)\Bigr),
\end{equation}
an inexact Krasnosel'ski\u{\i}--Mann iteration for $h$ whose errors are bounded by $\|\overline U_{k-1}\|$. The deterministic inexact Krasnosel'ski\u{\i}--Mann estimate of \citet{bravo2024stochastic}, applied with errors $e_k=\overline U_{k-1}$ exactly as in that proof, combined with $\|x_k-h(x_k)\|\le\|z_k-h(z_k)\|+2\|\overline U_k\|$ and taken in expectation, gives \eqref{eq:markov-inexact-km}. This step holds for any step sizes in $(0,1)$ and uses no martingale structure, so the Markovian dependence and the schedule enter only through $\omega_k$ and $\tau_k$. Lemma~\ref{lem:markov-averaged-noise} below gives $\omega_k\le A\sqrt{\alpha_k}$ for both schedules. Since $\tau_k\ge\tau_1>0$ and $\alpha_k\le2\alpha_{k+1}$, this implies \eqref{eq:markov-poisson-envelope}. Proposition~\ref{prop:centered-markov-two-phase} below then bounds the remaining critical-tail sums in \eqref{eq:markov-inexact-km}.
\end{proof}

\begin{lemma}[Markovian averaged noise]\label{lem:markov-averaged-noise}
In the setting of Lemma~\ref{lem:markov-inexact-km}, let $\|\cdot\|_2$ be a Euclidean norm with $\|\cdot\|\le\|\cdot\|_2\le c_E\|\cdot\|$, and suppose that the step sizes $\alpha_k\in(0,1)$ are non-increasing and satisfy $\alpha_k^{-1}-\alpha_{k-1}^{-1}\le\bar a$ for all $k\ge2$ and some $\bar a<1$. Then there is a constant $A$, depending only on $\bar a$, $\alpha_1$, $c_E$, the diameter of $K$, and the fundamental matrix of $(Y_k)$, such that $\E\|\overline U_k\|_2\le A\sqrt{\alpha_k}$ for all $k\ge1$.
\end{lemma}
Both schedules in Lemma~\ref{lem:markov-inexact-km} qualify. For $\alpha_k=(k+1)^{-a}$, $(k+1)^a-k^a\le a<1$. For the two-phase schedule, the first phase gives at most $a_1$, the transition gives $0$ because $\alpha_{T+1}=\alpha_T$, and in the tail $\gamma_T^{-1}\bigl((k+1)^{4/5}-k^{4/5}\bigr)\le\frac45\gamma_T^{-1}(T+1)^{-1/5}\le\frac45(T+1)^{a_1-1}<a_1$. For the categorical recursion, $c_E=\sqrt{|\X|}$. For the coupled recursion with $\|(u,g)\|_2:=\sqrt2\,(\|u\|_{\mathrm C,2}^2+\lambda^2g^2)^{1/2}$, $c_E=\sqrt{2|\X|}$.
\begin{proof}
Let $Q_Y$ be the transition kernel of $(Y_k)$ and $\pi_Y$ its stationary law, so $h=\sum_y\pi_Y(y)H(\cdot,y)$. Since $(Y_k)$ is finite, irreducible and aperiodic, its fundamental matrix $Z_Y:=(I-Q_Y+\mathbf1\pi_Y^\top)^{-1}$ exists, $Z_Y-\mathbf1\pi_Y^\top=\sum_{n\ge0}(Q_Y^n-\mathbf1\pi_Y^\top)$ has zero row sums, $\kappa_Y:=\|Z_Y-\mathbf1\pi_Y^\top\|_\infty<\infty$, and
\begin{equation}
u(x,y):=\sum_{y'}\bigl(Z_Y-\mathbf1\pi_Y^\top\bigr)(y,y')\,H(x,y')
\end{equation}
solves the Poisson equation $u-Q_Yu=H-h$. Because the rows of $Z_Y-\mathbf1\pi_Y^\top$ sum to zero and $H$ takes values in $K$, $\|u\|_2\le\kappa_Y\,\mathrm{diam}_2(K)=:B$, and the same holds for $Q_Yu$. Because each $H(\cdot,y')$ is non-expansive in $\|\cdot\|$, both $u(\cdot,y)$ and $Q_Yu(\cdot,y)$ are Lipschitz from $\|\cdot\|$ to $\|\cdot\|_2$ with constant $L:=c_E\kappa_Y$.

With $\mathcal F_k:=\sigma(Y_0,\dots,Y_k)$, the point $x_{k-1}$ is $\mathcal F_{k-1}$-measurable, and
\begin{equation}
\xi_k=m_k+(r_{k-1}-r_k)+e_k,
\end{equation}
where $m_k:=u(x_{k-1},Y_k)-(Q_Yu)(x_{k-1},Y_{k-1})$, $r_k:=(Q_Yu)(x_k,Y_k)$ and $e_k:=(Q_Yu)(x_k,Y_k)-(Q_Yu)(x_{k-1},Y_k)$. Here $\E[m_k\mid\mathcal F_{k-1}]=0$, $\|m_k\|_2\le2B$, $\|r_k\|_2\le B$, and $\|e_k\|_2\le L\|x_k-x_{k-1}\|\le LD\alpha_k$ with $D$ the $\|\cdot\|$-diameter of $K$.

Write $\overline U_k=\sum_{t=1}^kw_{k,t}\xi_t$ with $w_{k,t}:=\alpha_t\prod_{r=t+1}^k(1-\alpha_r)$, and let $v_k:=\sum_tw_{k,t}^2$ and $s_k:=\sum_tw_{k,t}\alpha_t$. Then $v_k=(1-\alpha_k)^2v_{k-1}+\alpha_k^2$ and $s_k=(1-\alpha_k)s_{k-1}+\alpha_k^2$. As in the proof of Proposition~\ref{prop:centered-iid-two-phase}, $v_k\le\alpha_k$. Moreover, $s_k\le\alpha_k/(1-\bar a)$ by induction: $s_1=\alpha_1^2\le\alpha_1$, and
\begin{equation}
\alpha_k-(1-\alpha_k)\alpha_{k-1}=\alpha_k\alpha_{k-1}\bigl(1-(\alpha_k^{-1}-\alpha_{k-1}^{-1})\bigr)\ge(1-\bar a)\alpha_k^2.
\end{equation}
We bound the three parts of $\overline U_k$ separately.

For the martingale part, orthogonality of martingale differences gives $\E\|\sum_tw_{k,t}m_t\|_2^2\le4B^2v_k\le4B^2\alpha_k$.

For the Lipschitz part, $\|\sum_tw_{k,t}e_t\|_2\le LD\,s_k\le LD\alpha_k/(1-\bar a)$.

For the telescoping part, summation by parts gives
\begin{equation}
\sum_{t=1}^kw_{k,t}(r_{t-1}-r_t)=w_{k,1}r_0-\alpha_kr_k+\sum_{t=1}^{k-1}(w_{k,t+1}-w_{k,t})\,r_t.
\end{equation}
Here $w_{k,1}\le s_k/\alpha_1$. For $t<k$, $w_{k,t+1}-w_{k,t}=\prod_{r=t+2}^k(1-\alpha_r)\,(\alpha_{t+1}-\alpha_t+\alpha_t\alpha_{t+1})$. Since $0\le\alpha_t-\alpha_{t+1}=\alpha_t\alpha_{t+1}(\alpha_{t+1}^{-1}-\alpha_t^{-1})\le\bar a\,\alpha_t\alpha_{t+1}$, the last factor lies in $[0,\alpha_t\alpha_{t+1}]$, and $\alpha_t\le(1+\bar a)\alpha_{t+1}\le2\alpha_{t+1}$. Hence $\sum_{t<k}|w_{k,t+1}-w_{k,t}|\le2s_k$, and the telescoping part has norm at most $B\bigl(1+(2+\alpha_1^{-1})/(1-\bar a)\bigr)\alpha_k$.

Combining the three parts and using $\alpha_k\le\sqrt{\alpha_k}$ proves the claim with $A:=3B+\bigl(LD+B(2+\alpha_1^{-1})\bigr)/(1-\bar a)$.
\end{proof}
Lemma~\ref{lem:markov-averaged-noise} is sharper than the envelope $\omega_k=O(\tau_k\sqrt{\alpha_{k+1}})$ of \citet[Lemma~C.4]{blaser2026asymptotic}. Their analysis allows unbounded iterates, so the Poisson solution can grow like $\tau_k$ along the trajectory. On the compact set $K$ it is uniformly bounded, which removes this factor.

\begin{proposition}[Two-phase Markovian residual refinement]\label{prop:centered-markov-two-phase}
Under the assumptions of Theorem~\ref{thm:centered-markov}, fix $a_1 \in \bigl(\frac45,1\bigr)$ and set
\begin{equation}
\varepsilon_{a_1}:=\frac{5a_1-4}{10},\qquad \kappa_{a_1}:=\frac{2^{\varepsilon_{a_1}}}{\log 2}.
\end{equation}
Let $T$ be the largest integer such that
\begin{equation}\label{eq:markov-two-phase-threshold}
(k+1)^{\varepsilon_{a_1}}\le \kappa_{a_1}\log(k+1),\qquad 1\le k\le T.
\end{equation}
Set $\gamma_T:=(T+1)^{-a_1}(T+2)^{4/5}$, and use the two-phase step size
\begin{equation}\label{eq:markov-two-phase-stepsize}
\alpha_k =
\begin{cases}
(k+1)^{-a_1}, & 1\le k\le T,\\
\gamma_T(k+1)^{-4/5}, & k>T.
\end{cases}
\end{equation}
Then there exists $C_{\mathrm{mk}} >0$ such that, for all $k\ge1$,
\begin{equation}
\E\Bigl[\ell^\Theta_{\mathrm C,\infty}\bigl(p_k,h_\mu(p_k)\bigr)\Bigr] \le C_{\mathrm{mk}}(k+1)^{-1/10} \min\Bigl\{\kappa_{a_1}\log(k+1),(k+1)^{\varepsilon_{a_1}}\Bigr\}.
\end{equation}
Consequently,
\begin{equation}
\E\Bigl[\ell^\Theta_{\mathrm C,\infty}\bigl(p_k,\mathcal G(p_k)\bigr)\Bigr] \le \frac{C_{\mathrm{mk}}}{\mu_{\min}}(k+1)^{-1/10} \min\Bigl\{\kappa_{a_1}\log(k+1),(k+1)^{\varepsilon_{a_1}}\Bigr\}.
\end{equation}
\end{proposition}
\begin{proof}
The threshold is well defined by the same argument as in Proposition~\ref{prop:centered-iid-two-phase}: the function $x^{\varepsilon_{a_1}}/\log x$ decreases from its value at $x=2$ until $x=\exp(1/\varepsilon_{a_1})$ and then increases to infinity. Thus the set of integers satisfying \eqref{eq:markov-two-phase-threshold} is a nonempty initial interval. The factor $\gamma_T$ makes the schedule continuous at the transition, since $\alpha_{T+1}=\alpha_T$, and hence non-increasing.

For $k\le T$, the recursion has only used the polynomial schedule $\alpha_t=(t+1)^{-a_1}$. The polynomial Markovian residual bound of \citet[Theorem~3.1]{blaser2026asymptotic} therefore gives
\begin{equation}
\E\Bigl[\ell^\Theta_{\mathrm C,\infty}\bigl(p_k,h_\mu(p_k)\bigr)\Bigr] \le C(k+1)^{-(1-a_1)/2}.
\end{equation}
Since
\begin{equation}
(k+1)^{-(1-a_1)/2} =(k+1)^{-1/10}(k+1)^{\varepsilon_{a_1}},
\end{equation}
and the minimum in the proposition is $(k+1)^{\varepsilon_{a_1}}$ in this phase, the desired bound follows for $k\le T$.

For $k>T$, we apply Lemma~\ref{lem:markov-inexact-km} in the cumulative-coordinate norm that represents the block-sup Cram\'er metric. Since $T$ is fixed once $a_1$ is fixed, the two-phase schedule gives constants $0<c_\tau<C_\tau<\infty$ such that, for all $k>T$,
\begin{equation}\label{eq:markov-tau-tail}
c_\tau(k+1)^{1/5}\le \tau_k\le C_\tau(k+1)^{1/5}.
\end{equation}
Indeed, the tail contribution is a constant multiple of $\sum_{t>T}(t+1)^{-4/5}$, and the finitely many transition values are absorbed into $c_\tau$ and $C_\tau$. Combining \eqref{eq:markov-tau-tail} with \eqref{eq:markov-poisson-envelope} yields
\begin{equation}\label{eq:markov-omega-tail}
\omega_k\le C(k+1)^{-1/5}.
\end{equation}
Consequently,
\begin{equation}\label{eq:markov-first-error-sum}
\sum_{t=2}^k \alpha_t\omega_{t-1} \le C\left(1+\sum_{t=T+1}^k (t+1)^{-1}\right) \le C\log(k+1).
\end{equation}
The first term in \eqref{eq:markov-inexact-km} is therefore bounded by $C(k+1)^{-1/10}\log(k+1)$, and the final correction term \eqref{eq:markov-omega-tail} is smaller than this bound.

The convolution term in \eqref{eq:markov-inexact-km} is handled by separating the finite prefix from the critical tail. The terms with $t\le T$ contribute at most $C(k+1)^{-1/10}$ after increasing $C$, again using \eqref{eq:markov-tau-tail}. The finitely many near-transition cases are absorbed into the same constant. For the tail terms $T<t<k$, \eqref{eq:markov-omega-tail} gives $\alpha_t\omega_{t-1}\le C(t+1)^{-1}$, and the tail lower bound gives
\begin{equation}
\tau_k-\tau_t\ge c\bigl((k+1)^{1/5}-(t+1)^{1/5}\bigr).
\end{equation}
Thus
\begin{equation}
\sum_{t=T+1}^{k-1}\alpha_t\sigma(\tau_k-\tau_t)\omega_{t-1} \le C\sum_{t=T+1}^{k-1} \frac{(t+1)^{-1}}{\sqrt{(k+1)^{1/5}-(t+1)^{1/5}}}.
\end{equation}
By the same integral comparison used in the i.i.d.\ proof, with the change of variables $u=(x+1)^{1/5}$, the last display is at most
\begin{equation}
C\int_{(T+1)^{1/5}}^{(k+1)^{1/5}} \frac{\mathrm d u}{u\sqrt{(k+1)^{1/5}-u}} \le C(k+1)^{-1/10}\log(k+1).
\end{equation}
The omitted $t=k$ endpoint is bounded by $C(k+1)^{-1}$ and is therefore harmless. Combining these bounds in \eqref{eq:markov-inexact-km} gives
\begin{equation}
\E\Bigl[\ell^\Theta_{\mathrm C,\infty}\bigl(p_k,h_\mu(p_k)\bigr)\Bigr] \le C(k+1)^{-1/10}\log(k+1),\qquad k>T.
\end{equation}
For $k>T$, the definition of the largest crossing $T$ implies
\begin{equation}
(k+1)^{\varepsilon_{a_1}}> \kappa_{a_1}\log(k+1),
\end{equation}
so the minimum in the proposition is $\kappa_{a_1}\log(k+1)$. Since $\kappa_{a_1}\ge1$, enlarging the constant proves the claimed residual bound in the secondary phase. The projected-operator residual bound follows from Corollary~\ref{corol:residual-conversion}.
\end{proof}

We are now in a position to state the proof of Theorem~\ref{thm:centered-markov}.
\begin{proof}[Proof of Theorem~\ref{thm:centered-markov}] 
Because $\X$ and $\A$ are finite and the reward function $R:\X\times\A\to[0,1]$ is deterministic, the centered reward $B_k$ is finite even after conditioning only on $(S_k,S_{k+1})$. Hence the feasible centered sample set
\begin{equation}
\{(i, b,j) \in \X \times \R \times \X: \mathbb P(S_k =i, B_k = b, S_{k+1}=j) >0 \}
\end{equation}
is finite. The process $(Y_k)_{k\ge0}$ is Markov because the conditional law of $(B_{k+1}, S_{k+2})$ depends only on $S_{k+1}$, which is the last coordinate of $Y_k = (S_k, B_k, S_{k+1})$. It is irreducible and aperiodic on its feasible support because from any feasible triple $(i, b,j)$ to any other feasible triple $(i', b', j')$, one may first follow the irreducible and aperiodic state chain from $j$ to $i'$ and then realize the feasible transition producing $(b', j')$. By Lemma~\ref{lem:one-sample-nonexp}, the sampled map is samplewise non-expansive, and by Proposition~\ref{prop:mean-field-fix} and Corollary~\ref{corol:fixed-points}, the averaged operator is $h_\mu$, which is non-expansive with non-empty fixed-point set. The hypotheses of \citet{blaser2026asymptotic} are therefore satisfied, giving the fixed-exponent a.s.\ convergence assertion and the residual bound $C/\sqrt{\tau_k}$ with $\tau_k=\sum_{t=1}^{k}\alpha_t(1-\alpha_t)$. The two-phase residual assertion is Proposition~\ref{prop:centered-markov-two-phase}. The projected residual bounds once again follow from Corollary~\ref{corol:residual-conversion}.
\end{proof}

\subsection{Proof of Proposition~\ref{prop:counterexample}}
\begin{proof}[Proof of Proposition~\ref{prop:counterexample}]
Consider a one-state deterministic MDP. The transition matrix is $(1)$, the reward is almost surely $r$, and therefore $\bar r^\pi = r$. For any categorical state $p \in \Delta_d$, the centered update uses the shift $r - \bar r^\pi =0$, so $\mathcal G(p) = L_0 p = p$. Hence, $\mathcal G = \mathrm{Id}$. Now replace the centered reward by the uncentered reward in the one-state sampled update. The resulting update sends $p$ to $L_r p$. If the support has at least two atoms and $r \ne 0$, then $L_r p \ne p$ in general. For example if $p$ is a point mass at an interior atom, translating by $r$ and projecting back to the same support changes the cumulative coordinates. Therefore, the uncentered update is not an unbiased sample of $\mathcal G$, and in general it does not define an SA of the projected quotient operator.
\end{proof}

\newpage
\section{Proofs of Section~\ref{sec:uncentered}}\label{app:C}
For convenience of notation, let us write
\begin{equation}
\widehat H\bigl((p, g), (s, r, s') \bigr) := \bigl(H_g(p, (s, r, s')), r \bigr)
\end{equation}
and
\begin{equation}
\hat h_\mu(p, g) := \bigl(h^{(g)}_\mu(p), \bar r^\pi \bigr).
\end{equation}
This notation is more intuitive in making the gain estimate $g$ explicitly an argument of the recursion.
\begin{lemma}\label{lem:appC-1}
For every $b, b' \in \R$ and all $p, q \in \Delta_d$,
\begin{equation}
\ell^\Theta_\mathrm C(L_b p, L_{b'}q) \le \ell^\Theta_\mathrm C(p, q) + \varsigma^{-1/2}\lvert b - b'\rvert.
\end{equation}
\end{lemma}
\begin{proof}
By the triangle inequality and the non-expansiveness of $L_{b'}$,
\begin{equation}
\ell^\Theta_{\mathrm C}(L_b p, L_{b'}q) \le \ell^\Theta_\mathrm C(L_b p, L_{b'}p) + \ell^\Theta_\mathrm C(L_{b'} p, L_{b'} q) \le \ell^\Theta_\mathrm C(L_b p, L_{b'}p) +\ell^\Theta_\mathrm C(p, q).
\end{equation}
It therefore suffices to bound $\ell^\Theta_\mathrm C(L_b p, L_{b'}p)$. Write $p = \sum_{m=1}^d p_m e_m$, where $e_m$ is the $m$th canonical basis vector. Since scalar categorical projection is linear on mixtures,
\begin{equation}
L_t p = \sum_{m=1}^d p_m L_t e_m\qquad\text{for every } t \in \R.
\end{equation}
Applying triangle inequality to the vector of cumulative masses yields
\begin{equation}\label{eq:sub-into-this}
\ell^\Theta_\mathrm C(L_b p, L_{b'}p) \le \sum_{m=1}^d p_m\ell^\Theta_\mathrm C(L_b e_m, L_{b'}e_m).
\end{equation}
Fix $m$ and define
\begin{equation}
v_m(t) := \bigl(F_{L_t e_m}(\theta_k)\bigr)_{k=1}^{d-1} \in \R^{d-1}.
\end{equation}
Since $\ell^\Theta_\mathrm C(u,v) = \sqrt{\varsigma}\bigl\lVert (F_u(\theta_k) - F_v(\theta_k))_{k=1}^{d-1}\bigr\rVert_2$, it is enough to control $t \mapsto v_m(t)$. Let $x = \theta_m+t$ be the translated atom location. If $x \le \theta_1$, then $L_t e_m = e_1$ and $v_m(t) = \mathbf 1$. If $x \ge \theta_d$, then $L_t e_m = e_d$ and $v_m(t) = 0$. If instead $x \in [\theta_k, \theta_{k+1}]$ for some $k \in \{1, \dots, d-1\}$, scalar categorical projection places mass
\begin{equation}
\phi_k(x) := \frac{\theta_{k+1} -x}{\varsigma}
\end{equation}
on $\theta_k$ and mass $1 - \phi_k(x)$ on $\theta_{k+1}$. Hence
\begin{equation}
v_m(t) = \bigl(\underbrace{0,\dots,0}_{k-1 \text{ entries}},\phi_k(x),\underbrace{1, \dots, 1}_{d-1-k\text{ entries}}\bigr),
\end{equation}
so on this interval $v_m$ is affine and
\begin{equation}
\frac{\mathrm d}{\mathrm dt} v_m(t) = -\frac{1}{\varsigma}e_k.
\end{equation}
Therefore $\bigl\lVert\frac{\mathrm d}{\mathrm dt} v_m(t) \bigr\rVert_2 \le \varsigma^{-1}$ wherever the derivative exists. Integrating this bound and multiplying by $\sqrt{\varsigma}$ from the definition of $\ell^\Theta_\mathrm C$ gives
\begin{equation}
\ell^\Theta_\mathrm C(L_b e_m, L_{b'}e_m) \le \varsigma^{-1/2}\lvert b- b'\rvert.
\end{equation}
Substituting this into \eqref{eq:sub-into-this} yields
\begin{equation}
\ell^\Theta_\mathrm C(L_b p, L_{b'}p) \le \varsigma^{-1/2}\lvert b- b'\rvert
\end{equation}
\end{proof}

\begin{lemma}\label{lem:appC-2}
For all $g, g' \in \R$ and all $p, q \in \Delta^\X_d$,
\begin{equation}
\ell^\Theta_{\mathrm C, \infty}\bigl(\mathcal G_g(p), \mathcal G_{g'}(q) \bigr) \le \ell^\Theta_{\mathrm C, \infty}\bigl(p, q\bigr) + \varsigma^{-1/2}\lvert g - g'\rvert.
\end{equation}
\end{lemma}

\begin{proof}
Fix $i \in \X$. By Proposition~\ref{prop:g-block-form},
\begin{equation}
\mathcal G_g(p)_i = \sum_{j\in \X} P^\pi_{ij} \E[L_{R_{ij} -g}p_j],\qquad \mathcal G_{g'}(q)_i = \sum_{j\in \X} P^\pi_{ij} \E[L_{R_{ij} -g'}q_j].
\end{equation}
For each $j$, let
\begin{equation}
u_j := \bigl( F_{L_{R_{ij} -g}p_j}(\theta_k) - F_{L_{R_{ij} -g'}q_j}(\theta_k)\bigr)^{d-1}_{k=1} \in \R^{d-1}.
\end{equation}
Since cumulative masses are linear in the coefficient vector, using triangle inequality,
\begin{equation}
\begin{aligned}
\ell^\Theta_\mathrm C\bigl( \mathcal G_g(p)_i, \mathcal G_{g'}(q)_i\bigr) &\le \sqrt{\varsigma} \bigl\lVert\sum_{j \in \X}P^\pi_{ij}\E[u_j]\bigr\rVert_2 \\
&\le \sum_{j \in \X} P^\pi_{ij} \E \bigl[\ell^\Theta_\mathrm C(L_{R_{ij} - g}p_j, L_{R_{ij}-g'}q_j) \bigr] \\
&\le \sum_{j\in \X} P^\pi_{ij} \E\bigl[\ell^\Theta_\mathrm C(p_j, q_j) + \varsigma^{-1/2}\lvert g- g' \rvert \bigr] \\
&\le \ell^\Theta_{\mathrm C, \infty}(p, q) + \varsigma^{-1/2}\lvert g - g'\rvert.
\end{aligned}
\end{equation}
Taking the maximum over $i \in \X$ proves the result.
\end{proof}

\subsection{Proof of Proposition~\ref{prop:uncentered-nonexpansion}}
\begin{proof}[Proof of Proposition~\ref{prop:uncentered-nonexpansion}]
Let $y = (s, r, s')$. The second coordinate of $\widehat H$ is the common reward sample $r$, so $\lambda\lvert r - r \rvert = 0$. For the first coordinate, only the block $s$ is updated, and Lemma~\ref{lem:appC-1} gives
\begin{equation}
\ell^\Theta_\mathrm C\bigl(H_g(p, y)_s, H_{g'}(q, y)_s \bigr) = \ell^\Theta_\mathrm C(L_{r -g}p_{s'}, L_{r-g'}q_{s'}) \le \ell^\Theta_\mathrm C(p_{s'}, q_{s'}) + \varsigma^{-1/2}\lvert g -g' \rvert.
\end{equation}
Taking the maximum over states,
\begin{equation}
\ell^\Theta_{\mathrm C, \infty}\bigl(H_g(p, y), H_{g'}(q,y) \bigr) \le \ell^\Theta_{\mathrm C, \infty}(p, q) + \varsigma^{-1/2}\lvert g- g'\rvert.
\end{equation}
Since $\lambda \ge \varsigma^{-1/2}$ by definition,
\begin{equation}
d_\lambda\bigl(\widehat H\bigl((p, g), y\bigr), \widehat H\bigl((q, g'), y\bigr)\bigr) \le \ell^\Theta_{\mathrm C, \infty}(p, q) + \varsigma^{-1/2}\lvert g - g'\rvert \le d_\lambda \bigl((p,g),(q,g') \bigr).
\end{equation}
\end{proof}

\subsection{Proof of Proposition~\ref{prop:mean-field-online}}
\begin{proof}[Proof of Proposition~\ref{prop:mean-field-online}]
The expectation identity follows from Proposition~\ref{prop:g-block-form} applied under the stationary one-step law:
\begin{equation}
\E\bigl[H_g(p, Y) \bigr] = h^{(g)}_\mu(p),\qquad \E[R] = \bar r^\pi.
\end{equation}
For non-expansiveness, use the same cumulative-distribution argument as in the proof of Proposition~\ref{prop:mean-field-fix}. Since
\begin{equation}
h^{(g)}_\mu(p)_i = (1-\mu_i)p_i + \mu_i\mathcal G_g(p)_i,\qquad h^{(g')}_\mu(q)_i = (1-\mu_i)q_i + \mu_i\mathcal G_{g'}(q)_i,
\end{equation}
we have for each $k=1,\dots,d-1$,
\begin{equation}
F_{h^{(g)}_\mu(p)_i}(\theta_k) - F_{h^{(g')}_\mu(q)_i}(\theta_k) = (1-\mu_i)\bigl(F_{p_i}(\theta_k) - F_{q_i}(\theta_k)\bigr) + \mu_i \bigl(F_{\mathcal G_g(p)_i}(\theta_k) - F_{\mathcal G_{g'}(q)_i}(\theta_k)\bigr).
\end{equation}
Triangle inequality gives, for each $i \in \X$,
\begin{equation}
\begin{aligned}
\ell^\Theta_\mathrm C\bigl(h^{(g)}_\mu(p)_i, h^{(g')}_\mu(q)_i\bigr) &\le (1-\mu_i) \ell^\Theta_\mathrm C(p_i, q_i) + \mu_i \ell^\Theta_\mathrm C\bigl(\mathcal G_g(p)_i, \mathcal G_{g'}(q)_i \bigr) \\
&\le (1-\mu_i)\ell^\Theta_{\mathrm C, \infty}(p, q) + \mu_i \bigl(\ell^\Theta_{\mathrm C, \infty}(p,q) + \varsigma^{-1/2}\lvert g -g'\rvert \bigr) \\
&\le \ell^\Theta_{\mathrm C, \infty}(p, q) + \varsigma^{-1/2}\lvert g- g' \rvert.
\end{aligned}
\end{equation}
Taking the maximum over $i \in \X$ yields 
\begin{equation}
\ell^\Theta_{\mathrm C, \infty}\bigl(h^{(g)}_\mu(p), h^{(g')}_\mu(q) \bigr) \le \ell^\Theta_{\mathrm C, \infty}(p, q) + \varsigma^{-1/2}\lvert g-g'\rvert,
\end{equation}
and therefore
\begin{equation}
d_\lambda\bigl(\hat h_\mu(p, g), \hat h_\mu(q, g')\bigr) \le \ell^\Theta_{\mathrm C, \infty}(p, q) + \varsigma^{-1/2}\lvert g - g' \rvert \le d_\lambda\bigl((p, g), (q,g')\bigr).
\end{equation}
For the statement about fixed points, if $(p, g) \in \mathrm{fix}(\hat h_\mu)$, then necessarily $g = \bar r^\pi$. The first coordinate then reads $p = h_\mu(p)$, and $\mu_{\min} >0$ implies $p \in \mathrm{fix}(\mathcal G)$ by Proposition~\ref{prop:mean-field-fix}. The converse implication is immediate.
\end{proof}

\subsection{Proof of Corollary~\ref{corol:residual-conv-online}}
\begin{proof}[Proof of Corollary~\ref{corol:residual-conv-online}]
For $i \in \X$,
\begin{equation}
\mu_i \ell^\Theta_{\mathrm C}\bigl(p_i, \mathcal G_g(p)_i\bigr) = \ell^\Theta_{\mathrm C}\bigl(p_i, h^{(g)}_\mu (p)_i\bigr).
\end{equation}
Also,
\begin{equation}
\ell^\Theta_{\mathrm C}\bigl(\mathcal G_g(p)_i, \mathcal G(p)_i\bigr) \le \varsigma^{-1/2}\lvert g - \bar r^\pi \rvert
\end{equation}
by Lemma~\ref{lem:appC-2} with $p = q$. Hence
\begin{equation}
\mu_i \ell^\Theta_{\mathrm C}\bigl(p_i, \mathcal G(p)_i\bigr) \le \ell^\Theta_{\mathrm C}\bigl(p_i, h^{(g)}_\mu(p)_i\bigr) + \varsigma^{-1/2}\lvert g - \bar r ^\pi\rvert.
\end{equation}
Taking the maximum over $i \in \X$,
\begin{equation}
\mu_{\min}\ell^\Theta_{\mathrm C,\infty}\bigl(p, \mathcal G(p)\bigr) \le \ell^\Theta_{\mathrm C, \infty}\bigl(p, h^{(g)}_\mu(p)\bigr) + \varsigma^{-1/2}\lvert g - \bar r^\pi\rvert \le d_\lambda\bigl((p, g), \hat h_\mu(p, g)\bigr),
\end{equation}
where the last step uses $\lambda \ge \varsigma^{-1/2}$. The gain-error bound is immediate from the definition of $d_\lambda$.
\end{proof}

\subsection{Proof of Theorem~\ref{thm:online-main}}
\begin{proof}[Proof of Theorem~\ref{thm:online-main}]
Let $z_k = (p_k, g_k) \in \mathcal Z$. The coupled SKM recursion \eqref{eq:coupled-skm} can be written as
\begin{equation}
z_{k+1} = z_k + \alpha_k\bigl( \widehat H(z_k, Y_k) - z_k\bigr).
\end{equation}
View each categorical block $p_i$ through its cumulative masses $(F_{p_i}(\theta_1),\ldots,F_{p_i}(\theta_{d-1}))$, so that $\ell^\Theta_{\mathrm C,\infty}$ becomes the block-sup norm of these cumulative-coordinate vectors. On the augmented coordinates $(u,g)$, use the product norm $\lVert(u, g)\rVert_\lambda := \lVert u \rVert_{\mathrm{C},\infty} + \lambda \lvert g \rvert$. Under this identification, $d_\lambda$ is the induced distance on the compact convex set $\mathcal Z$. By Proposition~\ref{prop:uncentered-nonexpansion}, the sampled map is samplewise non-expansive in this norm. By Proposition~\ref{prop:mean-field-online} and Corollary~\ref{corol:fixed-points}, the mean-field map induced by the stationary one-step law is non-expansive with a non-empty set of fixed-points. Since rewards are finite-valued, the uncentered chain $(Y_k)_{k\ge0}$ is finite-state on its feasible support. The fixed-exponent Markovian theorem of \citet{blaser2026asymptotic} therefore applies. The final two bounds then follow from Corollary~\ref{corol:residual-conv-online}. The two-phase residual bound follows by applying Proposition~\ref{prop:centered-markov-two-phase}'s argument on the compact product space $(\mathcal Z, d_\lambda)$, with $\widehat H$ in place of $H$ and $\hat{h}_\mu$ in place of $h_\mu$. The only constants that change are the product-space diameter and the finite-state Poisson constant of the uncentered chain $(Y_k)_{k \ge 0}$. Corollary~\ref{corol:residual-conv-online} then gives the projected-operator and gain-error bounds.
\end{proof}

\newpage
\section{Synchronous gain-free structure}\label{sec:synchronous}

The previous section handles raw rewards by estimating the gain online. We close with a complementary exact-law observation: if one forms a full synchronous sampled backup before choosing categorical representatives, then the centering constant disappears already at the quotient level. This is a stronger, samplewise invariance statement than the operator-level invariance of Proposition~\ref{prop:tg}, but it lives in the idealized space of quotient distributions rather than in the practical fixed-grid categorical representation. To make the statement precise, let
\begin{equation}
\mathbf Y = \bigl((S'_i, R_i)\bigr)_{i \in \X}
\end{equation}
be a synchronous one-step sample such that, for each $i \in \X$,
\begin{equation}
\mathbb P(S'_i = j) = P^\pi_{ij},\qquad R_i \mid (S'_i=j) \sim \mathrm{Law}(R_{ij}),
\end{equation}
with the coordinates sampled independently across $i$. Thus $\mathbf Y$ contains one sampled successor and reward for every state. For $\eta \in \mathcal F^\X$ and $g \in \R$, define the exact synchronous sample map by
\begin{equation}
\bigl(\mathcal H_g(\eta, \mathbf Y)\bigr)_i := (\tau_{R_i -g})_\#\eta_{S'_i},\qquad i\in \X.
\end{equation}
Changing $g$ changes each centered reward $R_i-g$ by the same additive amount. Since the update is synchronous, that same shift appears in every output block.
\begin{proposition}\label{prop:sync-gain-free}
For every synchronous sample $\mathbf Y$, every $g, g' \in \R$, and every $\eta \in \mathcal F^\X$,
\begin{equation}
\mathcal H_{g'}(\eta, \mathbf Y) = (\tau_{g-g'})_\# \mathcal H_{g}(\eta, \mathbf Y).
\end{equation}
Consequently,
\begin{equation}
[\mathcal H_g(\eta, \mathbf Y)] = [\mathcal H_{g'}(\eta, \mathbf Y)]\quad \text{in } \mathcal F^\X / \approx.
\end{equation}
Thus the exact synchronous quotient sample is independent of the centering constant.
\end{proposition}
\begin{proof}
For every state $i$, $R_i - g' = (R_i - g) + (g-g')$, so $(\tau_{R_i - g'})_\# \eta_{S'_i} = (\tau_{g-g'})_\# \bigl((\tau_{R_i -g})_\# \eta_{S'_i} \bigr)$. The translation $(\tau_{g-g'})_\#$ is common to every state block, so the two output families represent the same element in the equivalence class.
\end{proof}
Taking the representative with $g=\bar r^\pi$ connects this gain-free quotient sample to the exact centered operator:
\begin{equation}
\E[\mathcal H_{\bar r^\pi}(\eta, \mathbf Y)] = \mathcal T \eta.
\end{equation}
To see this, condition on the successor state $S'_i = j$. The $i$th block of $\mathcal H_{\bar r ^\pi}(\eta, \mathbf Y)$  is then $(\tau_{R_i - \bar r^\pi})_\# \eta_j$ whose conditional law is $\nu_{ij}\,\ast\,\eta_j$. Averaging over $j$ with weights $P^\pi_{ij}$ gives exactly the $i$th component of $\mathcal T\eta$ from \eqref{eq:avg-operator}.

Thus synchronous exact samples provide a gain-free sample representation of the quotient-law target. This does not contradict the need for the online-gain categorical recursion above: once we choose a fixed categorical grid, projecting a translated representative can change the coefficients, so the exact quotient symmetry is not automatically preserved by a practical fixed-grid categorical update.

\newpage
\section{Five-state experimental details}\label{app:tabular-experiment}
For the convergence experiment in Figure~\ref{fig:mrp_main}, we use deterministic state rewards $r=(0.90,0.10,0.55,0.80,0.05)$ and
\begin{equation}\label{eq:experiment-P}
P^\pi=\begin{bmatrix}.10&.45&.25&.15&.05\\.20&.20&.35&.15&.10\\.10&.15&.30&.35&.10\\.15&.10&.25&.25&.25\\.30&.15&.10&.15&.30\end{bmatrix}.
\end{equation}
The categorical grid is $[-5,5]$ with $d=321$ and $\varsigma=0.03125$. Every method runs for $10^{7}$ iterations over ten seeds from the uniform initialization $p_{0,i}=(1/d,\dots,1/d)$. The i.i.d.\ recursions sample $S_k$ uniformly ($\rho_i=1/5$), the Markovian and coupled recursions start from $S_0\sim\mu$, and the coupled recursion starts from $g_0=0$. The wrong-centering ablation uses i.i.d.\ samples and the i.i.d.\ schedule with $g=0$. The exact KM exponent is $0.75$. The i.i.d.\ recursion uses the two-phase schedule of Proposition~\ref{prop:centered-iid-two-phase} with $a_1=0.75$. The theoretically prescribed phase switch occurs beyond the plotted horizon. The Markovian and coupled recursions use fixed exponent $0.85$. Curves show medians and bands show interquartile ranges. The empirical rates quoted for Figure~\ref{fig:mrp_main}(a) are least-squares slopes of the log median residual against $\log k$ over $k\ge10^4$: $-0.33$ (i.i.d.), $-0.40$ (Markovian), and $-0.38$ (coupled). The state--atom kernel in \eqref{eq:state-atom-kernel} has $1605$ states, one closed class of size $1557$, and $48$ transient singleton classes, so Theorem~\ref{thm:geometric-target} implies that the categorical fixed point is unique.

For Figure~\ref{fig:mrp_distributions}, we instead use rewards $(0.95,0.78,0.50,0.22,0.05)$ and
\begin{equation}
P^\pi=\begin{bmatrix}.78&.17&.03&.01&.01\\.18&.72&.07&.02&.01\\.04&.08&.76&.08&.04\\.01&.02&.07&.72&.18\\.01&.01&.03&.17&.78\end{bmatrix}.
\end{equation}
The large self-transition probabilities keep the chain in high- or low-reward regions long enough for the transient return to depend strongly on the starting state. We use geometric continuation probability $\gamma=0.9$, a $401$-atom grid on $[-20,20]$, and $100{,}000$ exact Monte Carlo returns per starting state. The largest empirical overflow probability is $3\times10^{-4}$, the largest categorical endpoint mass is $2.7\times10^{-4}$, and the maximum Cram\'er discrepancy between the recursive categorical fixed point and the projected Monte Carlo law is $0.012$.

\section{Neural function approximation results}\label{sec:pendulum}

\begin{figure}[!htb]
\centering
\includegraphics[width=\textwidth]{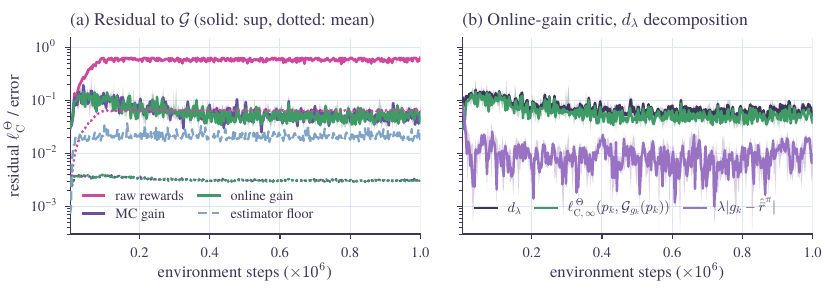}
\caption{
Continuous-state \texttt{Pendulum-v1} evaluation. Curves show medians over five seeds, and bands show interquartile ranges. \textbf{(a)} Monte Carlo estimate of the residual to the correctly centered operator $\mathcal G$ for the critics centered with the Monte Carlo gain, with the online gain, and with raw rewards. Solid curves take the supremum over $1000$ validation states, and dotted curves the average. The dashed blue curve is the Monte Carlo error of the estimator itself. \textbf{(b)} Online-gain critic: $d_\lambda$ with its categorical part $\ell^\Theta_{\mathrm C,\infty}(p_k,\mathcal G_{g_k}(p_k))$ and gain part $\lambda|g_k-\hat{\bar r}^\pi|$, $\lambda=\varsigma^{-1/2}$.
}
\label{fig:pendulum_residuals}
\end{figure}

We consider fixed-policy evaluation on \texttt{Pendulum-v1} \citep{gym}, with observations $(\cos\theta,\sin\theta,\dot\theta)$, torque $a\in[-2,2]$, and raw reward
\begin{equation}
r_{\mathrm{raw}}(\theta,\dot\theta,a) = -\theta^2 - 0.1\dot\theta^2 - 0.001a^2,
\end{equation}
normalized to $[0,1]$ by $(r_{\mathrm{raw}}+16.2736044)/16.2736044$. The fixed stochastic energy-based policy is
\begin{equation}
a_t = \mathrm{clip}\!\left(K_E \dot\theta_t \cos\theta_t + K_D \dot\theta_t + \sigma\epsilon_t, -2,2 \right), \qquad \epsilon_t\sim\mathcal N(0,1),
\end{equation}
with $K_E=1$, $K_D=0.2$ and $\sigma=0.7$.

The average-reward theory concerns one Markov chain, so all data come from a single trajectory whose stored transitions are exactly that chain's transitions. We remove the 200-step time limit and, after each step, restart from the environment's initial law ($\theta\sim U[-\pi,\pi]$, $\dot\theta\sim U[-1,1]$) with probability $1/200$, keeping the usual mean episode length. A restart is part of the stored transition and of the simulator used for the residual estimates. Without restarts the policy settles near the hanging position and the problem becomes nearly trivial. With the deterministic 200-step reset that \texttt{Pendulum-v1} uses by default \citep{gym}, the stored transitions, the state distribution and the Monte Carlo gain would refer to different chains.

Each categorical critic is a two-hidden-layer ReLU network of widths $(256,256)$ with a $2001$-way softmax on the support $[-200,200]$ ($\varsigma=0.2$). We train three of them on the same data, centered with a Monte Carlo gain estimate $\hat{\bar r}^\pi$ from a separate $200{,}000$-step trajectory, centered with the online gain $g_t$ (running average with coefficient $10^{-4}$, started at $0.5$), and with raw rewards. A scalar average-reward TD critic serves as a reference for the categorical mean. Adam learning rates are $10^{-4}$ (categorical) and $3\cdot10^{-4}$ (scalar), the replay buffer holds $300{,}000$ transitions, training starts after $5000$ steps, and each environment step performs two updates with minibatches of $512$. Target networks use Polyak coefficient $0.005$. A mean-centering penalty with coefficient $10^{-4}$ fixes the common shift. It is a practical regularizer and not part of the analyzed recursion. Each run lasts $10^6$ environment steps, and we report five seeds.

To measure convergence, every $2500$ steps the residual $\ell^\Theta_{\mathrm C}(p(s),\mathcal G(p)(s))$ is estimated at $1000$ validation states by averaging $128$ simulated one-step backups of the current critic, and reported as a supremum and an average over states. Splitting the $128$ backups into two halves gives an estimate of the Monte Carlo error of this estimator, shown as the estimator floor.

Over the last $20\%$ of training, the critics centered with the Monte Carlo and online gains reach supremum residuals of $0.04$--$0.07$ and average residuals of $0.003$, against an estimator floor of $0.02$ and $0.0017$, respectively. A substantial part of the remaining residual, and most of the average one, is therefore estimation noise. The raw-reward critic instead reaches $0.56$--$0.61$. Every sampled shift is nonnegative, so all of its mass ends on the upper endpoint, the same mechanism as the $g=0$ ablation in Figure~\ref{fig:mrp_main}(c). In panel (b), the weighted gain error falls from $0.08$--$0.09$ to $0.009$--$0.014$, and $d_\lambda$ is dominated by its categorical part. The plateaus reflect the constant learning rates. Since Theorem~\ref{thm:online-main} requires diminishing step sizes, the experiment illustrates it only qualitatively. The categorical means agree with the scalar critic ($R^2=0.99$), while the learned laws are wide (standard deviation about $125$) with about $3\%$ of their mass on the endpoints. This is consistent with Section~\ref{sec:meaning}, since in the undiscounted setting the spread of a projected fixed point is set by clipping, and the means remain the reliable output (Proposition~\ref{prop:poisson-consistency}).
% \newpage
% \input{checklist}
\end{document}